%% file: example_paper.tex
\documentclass{article}

\usepackage{microtype}
\usepackage{graphicx}
\usepackage{subcaption}
\usepackage{booktabs} %

\usepackage{hyperref}

\usepackage[accepted]{icml2026}

\usepackage{amsmath}
\usepackage{amssymb}
\usepackage{mathtools}
\usepackage{amsthm}

\usepackage[capitalize,noabbrev]{cleveref}

\usepackage{wrapfig}
\usepackage{outlines}

\usepackage{graphicx} 

\usepackage{enumitem}
\setlist{leftmargin=*}

\usepackage{amsmath}

\NewDocumentCommand{\increase}{m m o}{%
    \IfValueTF{#3}{%
        ${#1}_{\textcolor{red}{#2\uparrow}}^{#3}$ %
    }{%
        $#1_{\textcolor{red}{#2\uparrow}}$ %
    }%
}%
\NewDocumentCommand{\decrease}{m m o}{%
    \IfValueT{#3}{%
        $#1_{\textcolor{green}{#2\downarrow}}^{#3}$ %
    }{%
        $#1_{\textcolor{green}{#2\downarrow}}$ %
    }%
}%

\definecolor{ETHBlue}{RGB}{33,92,175}	%
\definecolor{ETHGreen}{RGB}{98,115,19}		%
\definecolor{ETHPurple}{RGB}{163,7,116}	%
\definecolor{ETHGray}{RGB}{111,111,111}	%
\definecolor{ETHRed}{RGB}{183,53,45}	%
\definecolor{ETHPetrol}{RGB}{0,120,148}	%
\definecolor{ETHBronze}{RGB}{142,103,19}	%
\definecolor{ETHPurpleLight}{RGB}{220, 158, 201}	%
\definecolor{ETHPurpleDark}{RGB}{140,10,89}	%

\newtoggle{color-macro}
\settoggle{color-macro}{false} %

\iftoggle{color-macro}{
\colorlet{MacroColor}{ETHPetrol}
}{
\colorlet{MacroColor}{black}
}

\newtheorem{definition}{Definition}

\newtheorem{theorem}{Theorem}

\newtheorem{assumption}{Assumption}

\usepackage{tcolorbox}
\usepackage{tikz}
\usetikzlibrary{patterns}
\usepackage{pgfplots}
\pgfplotsset{compat=1.18}
\usepgfplotslibrary{statistics} %
\usepackage{pgfplotstable}

\usepackage{scalefnt}
\usepackage{caption}
\usepackage[font=small,labelfont=bf]{subcaption}
\usepackage{anyfontsize}
\usepackage{multirow}
\usepgfplotslibrary{fillbetween}
\usepackage{changepage}
\usepackage{multicol}
\usepackage{setspace}
\usepackage[version=4]{mhchem}
\usepackage[colorinlistoftodos]{todonotes}

\usepackage{cleveref}
\crefname{section}{\S}{\S\S}
\Crefname{section}{\S}{\S\S}
\crefname{table}{Tab.}{Tabs.}
\crefname{figure}{Fig.}{Figs.}
\crefname{algorithm}{Alg.}{}
\crefname{appendix}{App.}{Apps.}
\crefname{lemma}{Lemma}{}
\Crefname{theorem}{Theorem}{}
\Crefname{assumption}{Assumption}{}
\crefname{proposition}{Proposition}{}
\crefname{hypothesis}{Hypothesis}{}
\crefname{deduction}{Deduction}{}
\crefname{intuition}{\textbf{Intuition}}{\textbf{Intuitions}}
\crefname{observation}{\textbf{Observation}}{\textbf{Observations}}
\crefname{finding}{\textbf{Finding}}{\textbf{Findings}}
\crefname{cor}{Corollary}{}
\crefname{align}{}{}
\crefname{equation}{}{}

\usepackage{xurl}

\icmltitlerunning{Diversity Matters: Revisiting Test-Time Compute in Vision-Language Models}

\begin{document}

\twocolumn[
  \icmltitle{Diversity Matters: Revisiting Test-Time Compute in Vision-Language Models}

  \icmlsetsymbol{equal}{*}

    \begin{icmlauthorlist}
    \icmlauthor{Yijie Tong}{equal,ethz}
    \icmlauthor{Yifan Hou}{equal,ethz}
    \icmlauthor{Shaobo Cui}{sjtu,epfl}
    \icmlauthor{Antoine Bosselut}{epfl}
    \icmlauthor{Mrinmaya Sachan}{ethz}
    \\
    $\{$
    \texttt{\href{mailto:yijie.tong@uzh.ch}{yijie.tong},}
    \texttt{\href{mailto:yifan.hou@inf.ethz.ch}{yifan.hou},}
    \texttt{\href{mailto:mrinmaya.sachan@inf.ethz.ch}{mrinmaya.sachan}}
    $\}$\texttt{@inf.ethz.ch}
    \\ 
    \texttt{\href{mailto:shaobo.cui@sjtu.edu.cn}{shaobo.cui@sjtu.edu.cn},}
    \texttt{\href{mailto:antoine.bosselut@epfl.ch}{antoine.bosselut@epfl.ch}}
  \end{icmlauthorlist}

  \icmlaffiliation{ethz}{ETH Zürich}
  \icmlaffiliation{sjtu}{Shanghai Jiao Tong University}
  \icmlaffiliation{epfl}{EPFL}

  \icmlcorrespondingauthor{Yifan Hou}{yifan.hou@inf.ethz.ch}
  \icmlcorrespondingauthor{Mrinmaya Sachan}{mrinmaya.sachan@inf.ethz.ch}

  \icmlkeywords{Machine Learning, ICML}

  \vskip 0.3in
]

\printAffiliationsAndNotice{\icmlEqualContribution}

\begin{abstract}
Test-time compute (TTC) strategies have emerged as a lightweight approach to boost reasoning in large language models (LLMs). However, their application and benefits for vision-language models (VLMs) remain underexplored.
We present a systematic study of TTC across seven VLMs and six benchmarks, specifically analyzing feature-based scoring and majority voting methods.
We find that feature heuristics fail and voting yields only modest gains in single-model settings.
We theoretically show that this limitation stems from a lack of \emph{prediction diversity}:
when outputs are highly correlated, voting provides little benefit. 
In contrast, multi-model ensembles offer richer diversity, yet standard majority voting fails to account for varying model capabilities.
To address this, we propose \emph{Entropy-based TTC} (ETTC), which selects the most confident prediction based on predictive entropy. 
Our method reduces to majority voting in the single-model case, but in model ensembles, it leverages confidence disparities to prioritize stronger models. 
We prove that ETTC outperforms majority voting under mild assumptions and empirically demonstrate that it consistently surpasses both voting and the best individual model.
Crucially, our results show that smaller models can synergistically enhance larger ones, unlocking ensembling gains not achievable with standard strategies.\footnote{Our code is publicly available here: \url{https://github.com/nanfang-wuyu/Diversity-Matters}.}
\end{abstract}

\section{Introduction}\label{sec:intro}
Vision-Language Models (VLMs) have recently achieved remarkable performance across a range of visual reasoning benchmarks~\citep{dubey24_arxiv_llama32, agrawal24_arxiv_pixtral12b, kamath25_arxiv_gemma3, bai25_arxiv_qwen25vl, openai23_arxiv_gpt4v, comanici25_arxiv_gemini25pushingfrontier}. 
At the same time, the large language modeling (LLM) community has developed a family of \emph{test-time compute} (TTC) strategies, particularly those based on \emph{chain-of-thought} (CoT) prompting, to improve reasoning without modifying model parameters~\citep{snell24_arxiv_ttc_deepmind}. 
These strategies generate multiple outputs per input and then aggregate or rank them to produce more reliable predictions.

In the LLM literature, inference-only TTC methods fall broadly into two paradigms.
The first is \emph{response selection} (often framed as Best-of-N), which scores and selects the most promising reasoning trace. While many of these methods rely on trained reward models, \emph{feature-based} approaches avoid training by estimating quality through textual heuristics. These include analyzing structural signals such as the presence of specific pivot words~\citep{chang25_arxiv_pivotword_alternatively,lippmann25_arxiv_pivotword_substance}, confident linguistic tone~\citep{DBLP:conf/emnlp/JiangLLBZSWL25,mao25_arxiv_cottongue}, or the length of the reasoning chain~\citep{fu23_iclr_cotlen_msreasoing,jin24_acl_cotlenhelps}.
In contrast, \emph{confidence-based} (e.g., self-consistency) methods treat the model as a stochastic oracle and improve reasoning reliability by aggregating multiple outputs, typically selecting the most frequent answer across samples via majority voting~\citep{wang23_iclr_selfconsis,chen24_acl_selfconsis_para,snell24_arxiv_ttc_deepmind}.

Applying TTC to VLMs, however, is far from straightforward. Unlike LLMs, VLMs must first perceive and interpret dense visual signals before reasoning over them. This introduces new challenges: 
(i)~visual perception is inherently error-prone and varies across models~\citep{bhattacharyya24_iclr_visualgrounding, wang25_arxiv_visualgrounding}; 
(ii)~vision-language alignment remains imperfect, creating subtle inconsistencies~\citep{li25_arxiv_vtalignmentmi, yan25_acl_vtalignmentbenchmark}; and 
(iii)~textual cues that correlate with correctness in LLMs may not reflect true visual understanding~\citep{altahan24_arxiv_textbutnotvisual, jiang25_arxiv_cotquality4vlm}. 
Therefore, it is unclear whether and when TTC strategies can reliably enhance visual reasoning.

To investigate this, we begin with the \emph{single-model (multi-round)} setting, where one VLM is queried multiple times with some notion of randomness (\cref{sec:evaluation}). 
Our findings reveal that:
(1)~feature-based methods fail to improve accuracy, showing that linguistic style is a poor proxy for visual reasoning quality; and 
(2)~confidence-based methods such as majority voting provide only modest, but consistent, gains, and only when CoT prompting is used. Without CoT, even aggregation brings no benefit.

Next, we investigate why these gains are so limited. 
Specifically, we analyze the \textit{diversity} (formally, the \textit{statistical dependency}) between predictions and show that the effectiveness of voting decreases as predictions become more correlated (\cref{sec:when:theory}). 
When model outputs are nearly identical, voting cannot amplify the signal of correctness. 
Empirically, we confirm this across 7 VLMs and 6 datasets: outputs exhibit weak but nonzero dependency, which explains why voting offers only small improvements in practice (\cref{sec:when:exp}).

These insights point to a deeper limitation: in the single-model setting, diversity arises only from sampling randomness, so the expected skill of the model remains unchanged. 
By contrast, the \emph{multi-model ensemble} setting naturally introduces stronger diversity: differences in architecture, training data, and even scale create complementary strengths. 
This makes ensembles both more realistic in practice and more promising for TTC. 
Existing methods, such as majority voting, cannot exploit this potential: by treating all models equally, voting risks letting weaker but correlated models dominate the outcome. 
What is needed is a strategy that adapts to model quality and selectively prioritizes the most reliable predictions.

To address this, we introduce a new strategy for visual reasoning: \emph{Entropy-based TTC (ETTC)} (\cref{sec:ettc:method}). 
Instead of counting votes, ETTC selects the prediction with the lowest entropy (on the answer distribution from multiple responses), that is, the most confident output distribution. 
In the single-model setting, ETTC reduces to majority voting, ensuring backward compatibility. 
But in multi-model ensembles, ETTC diverges from standard voting: it leverages confidence gaps across models, allowing smaller models to assist stronger ones rather than overwhelm them. 
We theoretically prove that ETTC outperforms majority voting under mild dependence assumptions (\cref{sec:ettc:theory}), and empirically show that it not only improves over voting but can even surpass the best individual model in the ensemble (\cref{sec:ettc:results}). 
This result is particularly striking: \emph{smaller models can enhance larger ones when combined wisely}, yielding gains not achievable with voting alone.

\noindent In summary, our contributions are:
\begin{itemize}[leftmargin=*]
    \item A systematic theoretical and empirical study of TTC in VLMs, showing that feature cues fail and that majority voting yields only modest CoT-dependent gains (\cref{sec:evaluation}).
    \item A theoretical analysis linking the effectiveness of voting to prediction dependency, supported by empirical evidence across diverse models and datasets (\cref{sec:when}).
    \item A new TTC strategy that generalizes majority voting and achieves consistent improvements in multi-model ensembles, often surpassing even the best single model (\cref{sec:ettc}).
\end{itemize}

\input{figures/results_ttc}

\section{Preparation}
We begin by outlining the models, datasets, prompting templates, baselines, and evaluation settings in our experiments.

\paragraph{Models.}
We evaluate seven open-source VLMs under two complementary multi-model ensemble configurations.  
In the \emph{similar-size (cross-family)} setup, we include four VLMs with comparable parameter sizes but diverse architectures: Qwen2.5-VL-7B-Instruct~\citep[Qwen-7B]{bai25_arxiv_qwen25vl}, LLaMA-3.2-11B-Vision~\citep[LLaMA]{dubey24_arxiv_llama32}, Gemma-3-12B-it~\citep[Gemma]{kamath25_arxiv_gemma3}, and Pixtral-12B-2409~\citep[Pixtral]{agrawal24_arxiv_pixtral12b}.  
In the \emph{same-family (varied-size)} setup, we use four models from the Qwen2.5-VL-Instruct family~\citep{bai25_arxiv_qwen25vl}, ranging from 3B to 72B parameters (3B, 7B, 32B, 72B), allowing us to study scaling effects within a single model architecture.

\paragraph{Datasets.}
We experiment on six multiple-choice visual QA benchmarks covering three domains.  
For \emph{mathematical reasoning}, we use the testmini split of MathVista~\citep{lu24_iclr_mathvista} and the test set of MathVision~\citep{wang24_nips_mathvision}.  
For \emph{diagram understanding}, we include the test sets of TQA~\citep{kim19_acl_tqa} and ScienceQA~\citep{lu22_nips_scienceqa}.  
For \emph{general visual reasoning}, we use the validation splits of MMStar~\citep{chen24_nips_mmstar} and MMMU~\citep{yue24_cvpr_mmmu}.  
All datasets contain multiple-choice questions with $K$ answer options ($2 \leq K \leq 9$).  
Further statistics, including domain, split size, and option counts, are summarized in \cref{tab:dataset_stats} of \cref{app:setting:dataset}.

\paragraph{Decoding.}
We generate responses using stochastic decoding~\citep{sutskever14_nips_decodinggs} via default settings.\footnote{\href{https://huggingface.co/docs/transformers/en/generation_strategies}{HuggingFace's Default Generation Strategies}}  
We adopt two prompting formats:  
(1) \emph{Direct Answer} prompting discourages intermediate reasoning and elicits immediate answers; 
(2) \emph{Chain-of-thought (CoT)} prompting explicitly encourages step-by-step reasoning, followed by a final answer.  
We use zero-shot, one-stage prompting for both settings to ensure consistency across models. Full prompt templates are provided in \cref{fig:prompt_noct,fig:prompt_cot} of \cref{app:setting:prompt}.  
Final answers are extracted from the text using regular expressions.

\paragraph{TTC Baselines.}
To revisit test-time compute strategies for visual reasoning, we evaluate four representative baselines spanning the feature-based selection and confidence-based aggregation paradigms.
Three are \textit{feature-based} Best-of-$N$ methods that score and rank CoT responses using lexical heuristics, rather than requiring a trained reward model:
(1) \emph{CoT Pivot Word} ranks each response by counting predefined reasoning-related expressions (e.g., ``alternatively'')~\citep{chang25_arxiv_pivotword_alternatively,lippmann25_arxiv_pivotword_substance}; see the full phrase list in \cref{tab:pivot} of \cref{app:baselines}.
(2) \emph{CoT Length} prefers longer responses, following prior work suggesting a correlation between length and reasoning quality~\citep{fu23_iclr_cotlen_msreasoing,jin24_acl_cotlenhelps}.
(3) \emph{Feature-All} combines four interpretable features (pivot word count, vague word count, total token count, and lexical diversity) to compute a composite score (see \cref{tab:featureall}).
As a \textit{confidence-based} method, (4) \emph{Majority Voting}~\citep{wang23_iclr_selfconsis,chen24_acl_selfconsis_para,snell24_arxiv_ttc_deepmind} aggregates $N = 16$ samples and selects the most frequent final answer.

\paragraph{Evaluation Settings.}
We assess all test-time compute methods under two settings:  
(1) In the \emph{single-model (multi-round)} setting, a single VLM is queried $N$ times per question with stochastic decoding. Test-time compute is used to aggregate these intra-model outputs.  
(2) In the \emph{multi-model ensemble} setting, $M$ distinct VLMs are queried per question (each with multiple samples), introducing both intra- and inter-model variation.  
This setting allows us to study cross-model complementarity and test whether aggregating weaker models can improve upon the best individual model.

\section{Whether TTC Works in Visual Reasoning}
\label{sec:evaluation}
We begin by revisiting whether TTC strategies, widely used in LLMs, improve visual reasoning in VLMs. 
We evaluate four representative methods across six multiple-choice visual benchmarks and compare their performance under two prompting conditions: direct answering and CoT. Results are averaged across seven VLMs.

\paragraph{Direct Answer Prompt: TTC fails without CoT.}
The \emph{Direct Answer} setting tests whether test-time variation alone, without prompting explicit reasoning, can boost accuracy. Since no reasoning chains are produced, only confidence-based methods like majority voting are applicable.

As shown in \cref{fig:results_ttc} (left, direct answer), voting provides negligible or no improvement over the greedy baseline (often $<1\%$). Although we sample 16 outputs per question with stochastic decoding, the model's predictions remain mostly identical. 
This suggests that without CoT prompting, VLMs tend to output the same surface-level answer, showing little diversity in interpretation. 
Consequently, TTC offers no benefit under direct answering. This aligns with findings in LLMs~\citep{wang23_iclr_selfconsis,snell24_arxiv_ttc_deepmind}, but the issue is further exacerbated in VLMs by the perception bottleneck: visual content must first be accurately grounded before any meaningful reasoning variation can emerge.

\paragraph{CoT Prompt: Confidence helps, features do not.}
In contrast, when models are prompted to reason step-by-step, test-time strategies have room to work. This setup enables both feature-based scoring (e.g., using CoT length) and confidence-based aggregation (e.g., majority voting).

As shown in \cref{fig:results_ttc} (right, CoT), voting consistently improves performance across all benchmarks, with average gains of 2-4\%. This validates the utility of test-time sampling under CoT: the model explores diverse reasoning paths and occasionally corrects itself. 
However, the improvements are modest, suggesting that sampled CoTs are still highly correlated, a hypothesis we will formally investigate in \cref{sec:when}.
Meanwhile, feature-based methods fail to provide any consistent gain over vanilla CoT. Their performance often fluctuates slightly around the baseline. This highlights a key difference from LLMs: in VLMs, textual heuristics are poor proxies for reasoning correctness because visual understanding is the bottleneck. If perception fails, even a well-formed CoT cannot save the answer.

\paragraph{Takeaway.}
TTC can improve visual reasoning, but only under specific conditions. 
Without CoT prompting, models produce nearly identical outputs, leaving no room for improvement. 
Even with CoT, gains from voting are modest, and feature-based scoring fails to help, highlighting the unique challenges of visual reasoning where perception quality limits downstream reasoning.
This raises a key question: \emph{when does TTC actually help?} 
To answer this, we now turn to the analysis of majority voting, focusing on how its effectiveness depends on the statistical dependencies among model predictions.

\section{When Does TTC Work?}
\label{sec:when}
Why does test-time compute (TTC), especially majority voting, sometimes fail to improve accuracy in visual reasoning? 
We address this question by analyzing how the statistical dependency among model predictions influences the effectiveness of voting. 
To this end, we develop a theoretical framework that quantifies this relationship and support it with empirical evidence.

\subsection{Theoretical Insight: TTC Helps with Diverse Predictions}
\label{sec:when:theory}

\paragraph{Intuition.}
Before diving into formal definitions, the core intuition is simple: if a model makes the exact same mistake every time you ask it a question, taking a vote across multiple attempts will not fix the error. Voting only amplifies accuracy if the model's responses are somewhat diverse (i.e., not completely dependent on one another) but lean toward the correct answer on average. We formalize this by defining the ``dependency'' between predictions and proving that as dependency goes up, the benefit of voting goes down.

\paragraph{Setup.}
Suppose we have a multiple-choice question with $K$ options and a single true answer, denoted as $Y \in [K]$. We gather a total of $U$ predictions, labeled $X_1, \dots, X_U$. These can come from multiple decoding rounds of a single VLM or from different VLMs in an ensemble.\footnote{Note that our theoretical result holds in both conditions.} 
We evaluate whether a specific prediction is correct using a binary indicator $Z_u := \mathbb{I}\{X_u = Y\}$, and define the model's expected accuracy on a single trial as $p := \mathbb{E}[Z_u]$.

To represent majority voting mathematically, we count the total votes for each option $k$, denoted as $S_k := \sum_{u=1}^U \mathbb{I}\{X_u = k\}$. The final voting prediction is the option with the most votes, $\widehat{Y}_{\mathrm{MV}} := \arg\max_k S_k$. Finally, we define the overall accuracy of voting as $A_{\mathrm{MV}}(U) := \mathbb{P}(\widehat{Y}_{\mathrm{MV}} = Y)$, and the net improvement it provides over a single guess as $\Delta A_{\mathrm{MV}}(U) := A_{\mathrm{MV}}(U) - p$.

\paragraph{Dependency metrics.}
To measure how heavily the $U$ predictions rely on each other, we quantify their \textit{dependency} using two standard statistical metrics: \emph{normalized mutual information (NMI)} for the raw answer options, and \emph{correlation} for the correctness indicators.

For any two predicted answers $X$ and $X'$, NMI measures the shared information between them, normalized by their individual uncertainty (entropy, $H$):
  \[
  \mathrm{NMI}(X; X') := \frac{I(X; X')}{\min\{H(X), H(X')\}},
  \]
  \[
  H(X) = -\sum_{k=1}^K \mathbb{P}(X = k) \log \mathbb{P}(X = k).
  \]
For the full set of $U$ predictions, the average NMI across all pairs is:
  \[
  \overline{\mathrm{NMI}} := \frac{2}{U(U-1)} \sum_{u < v} \mathrm{NMI}(X_u; X_v).
  \]

Similarly, for any two correctness indicators $Z$ and $Z'$, we define their statistical correlation (where $p$ is the single-trial accuracy), and average it across all pairs:
  \[
  \rho(Z, Z') \!:=\! \frac{\mathbb{E}[ZZ'] - p^2}{p(1 - p)}, \;\;
  \overline{\rho} \!:=\! \frac{2}{U(U - 1)} \sum_{u < v} \rho(Z_u, Z_v).
  \]

\begin{theorem}
\label{thm:mv-dependency}
Suppose all prediction pairs $(X_u, X_v)$ share the same dependency level (i.e., $\overline{\mathrm{NMI}}$ or $\overline{\rho}$). Then the voting improvement $\Delta A_{\mathrm{MV}}(U)$ is monotonically decreasing in both $\overline{\rho}$ and $\overline{\mathrm{NMI}}$. In particular:
\[
\overline{\rho} = 1\ (\text{or } \overline{\mathrm{NMI}} = 1)
\Rightarrow\ \Delta A_{\mathrm{MV}}(U) = 0,
\]
\[
\overline{\rho} \!=\! 0\ (\text{or } \overline{\mathrm{NMI}} \!=\! 0),\ p \!>\! \tfrac{1}{K}
\Rightarrow\ A_{\mathrm{MV}}(U) \!\to\! 1\ \text{as } U \!\to\! \infty.
\]
\end{theorem}

\paragraph{Interpretation.}
The formal proof is in \cref{app:proof:thm:mv-dependency}. For practical application, this theorem reveals a powerful boundary condition for test-time compute: \emph{voting only improves accuracy when predictions are diverse}. 
If all predictions are identical (correlation equals 1), voting reduces to a single prediction, yielding zero gain. Conversely, if predictions are entirely uncorrelated and individually better than random guessing, voting can aggregate the faint signals to achieve near-perfect accuracy given enough samples. Because NMI and correlation are model-agnostic, they serve as highly practical tools to estimate whether TTC will actually help on a given task, without needing access to ground truth labels.

\subsection{Empirical Verification}
\label{sec:when:exp}
To validate our theoretical insight, we structure our empirical evaluation into two parts. First, we determine the practical minimum number of decoding samples required to reliably estimate prediction dependency. Second, we test the core hypothesis: does the benefit of voting genuinely decrease as models become more correlated?

\paragraph{Motivation.}
Our theoretical analysis assumes a sufficiently large number of decoding samples $U$ for the benefits of voting to fully materialize. In practice, generating many samples incurs steep computational costs. Therefore, we first investigate how quickly our dependency metrics converge as $U$ grows, aiming to find the minimal sample size that yields stable estimates.

\paragraph{Setup.}
We use Qwen-7B to generate $U=2$ to $16$ decoded outputs for each example across six visual reasoning datasets. For each subset size $U$, we compute the average normalized mutual information ($\overline{\mathrm{NMI}}$) and average correctness correlation ($\overline{\rho}$) between response pairs.

\input{figures/convergence_nimcorr}

\paragraph{Findings.}
As shown in \cref{fig:convergence_nimcorr:qwen7b}, both $\overline{\mathrm{NMI}}$ and $\overline{\rho}$ stabilize rapidly, flattening out around $U=12$ across all datasets. Beyond this point, drawing additional samples offers minimal benefit in accurately estimating prediction dependency.

\paragraph{Takeaway.}
Sampling more than 12 to 16 responses provides diminishing returns. To ensure both statistical stability and computational tractability, we confidently use $U=16$ in all subsequent experiments.

\subsubsection{Does Voting Improvement Decrease with Dependency?}

\paragraph{Motivation.}
We now test our central theoretical claim: voting is most beneficial when model outputs are diverse, meaning the accuracy gained from voting should measurably shrink as prediction dependency increases.

\paragraph{Setup.}
We evaluate the voting improvement $\Delta A_{\mathrm{MV}}(16)$ for seven models across six datasets, utilizing $U\!=\!16$ decoding samples per query. For each model-dataset pair, we compute the average accuracy improvement alongside its average dependency (measured by both $\overline{\mathrm{NMI}}$ and $\overline{\rho}$).

\paragraph{Findings.}
\cref{fig:improvement_both:overview} shows a clear, consistent negative correlation between the improvement gained from voting and both dependency metrics. Smaller models (e.g., Qwen-3B, LLaMA) inherently produce more diverse outputs and therefore reap larger benefits from voting. In contrast, larger or more heavily optimized models (e.g., Qwen-72B, Pixtral) exhibit highly deterministic, low-diversity behavior, resulting in minimal gains from test-time aggregation. Detailed results broken down by dataset are in \cref{fig:improvement_nmi:all,fig:improvement_corr:all} in \cref{app:results:mv_correlation}.

\paragraph{Takeaway.}
The effectiveness of majority voting hinges entirely on the diversity of model outputs. As predictions become more deterministic (higher $\overline{\mathrm{NMI}}$ and $\overline{\rho}$), voting offers sharply diminishing returns. This establishes a practical principle for deploying TTC: voting is most beneficial when applied to weaker/smaller models, or in uncertain scenarios (like few-shot tasks or domain shifts) where outputs are naturally more stochastic. Conversely, wrapping large, highly consistent models in a voting ensemble often wastes compute for negligible gain.

\input{figures/improvement_nmicorr}

\section{Beyond Voting: Entropy-Based TTC for Multi-Model Ensembles}
\label{sec:ettc}
Building on the insight that majority voting benefits from diverse yet independent predictions, we now turn to the more realistic and underexplored \textit{multi-model ensemble} setting. 
Compared to multi-round decoding from a single model, which suffers from limited prediction diversity, ensembles of heterogeneous models naturally offer complementary strengths and errors. 
We first introduce an Entropy-based TTC method (ETTC) designed to better leverage cross-model diversity. 
Then, we theoretically show that ETTC outperforms majority voting under mild assumptions, and empirically demonstrate that it enables smaller models to reliably enhance larger ones in visual reasoning tasks.

\subsection{Entropy-Based TTC (ETTC)}
\label{sec:ettc:method}
Our previous analysis showed that the effectiveness of voting depends heavily on prediction diversity. However, majority voting has a deeper limitation in multi-model ensemble settings: it assumes all model responses are equally reliable and votes based solely on frequency, ignoring how confident or capable each individual model is.
This oversight is less problematic in the single-model setting, since all predictions come from the same source and share the same expected quality. 
But in multi-model ensembles, where models vary drastically in size, training, and performance, this uniform treatment becomes a liability. 
A majority of weaker models can easily outvote a stronger one, even when the strong expert is confidently correct.

\paragraph{Intuition.}
To fix this, we need a mechanism that listens to the ``expert'' in the room for any given question. Intuitively, when a capable model knows the correct answer, its output probability distribution will be highly concentrated (low uncertainty). Conversely, when a weaker model guesses, its distribution will be flatter (high uncertainty). Therefore, instead of counting votes, we can select the answer from the model that is most certain.
To operationalize this, we introduce Entropy-based Test-Time Compute (ETTC): a simple, model-agnostic method that selects the most confident prediction among multiple sources using normalized predictive entropy as a proxy for uncertainty.

\begin{definition}[Entropy-Based Selection Rule]
Let $U$ sources (e.g., different models in an ensemble) each produce a predictive distribution $p_u(\cdot) \in \Delta^{K-1}$ over $K$ answer options. We define the normalized entropy of model $u$ as:
\[
\widetilde{H}_u := -\frac{1}{\log K}\sum_{k=1}^K p_u(k)\log p_u(k) \in [0,1],
\]
and its top-1 prediction as $\hat{y}_u := \arg\max_{k} p_u(k)$. ETTC simply selects the prediction from the least-uncertain source:
\[
u^\star := \arg\min_{u\in[U]} \widetilde{H}_u,
\qquad
\widehat{Y}_{\min H} := \hat{y}_{u^\star}.
\]
\end{definition}
This selection rule prioritizes predictions with lower uncertainty. 
In contrast to majority voting, which can amplify weak or erroneous signals through sheer numbers, ETTC amplifies precision by trusting the most decisive prediction.
Notably, ETTC safely reduces to standard voting in the single-model case (average the predictive distributions over multiple rounds and pick the most probable option). But in the multi-model setting, it diverges: it allows strong models to dominate the decision even when they are in the minority, which is essential for leveraging heterogeneous ensembles.

\paragraph{Takeaway.}
ETTC replaces raw vote counts with model confidence, providing a more principled and adaptive aggregation strategy. In real-world scenarios where model capabilities vary, ETTC prevents over-reliance on weaker models while fully exploiting the reliability of stronger ones.

\subsection{Theoretical Insight: ETTC Outperforms Voting in Ensembles}
\label{sec:ettc:theory}
In a multi-model ensemble, differences in training data and architecture naturally increase answer diversity. While voting treats all models equally, this can backfire: weaker models may collectively outvote stronger ones, especially when their errors are correlated (e.g., due to shared pre-training data). 
Our goal is to theoretically prove why ETTC provides a more robust alternative in such correlated scenarios.

\paragraph{Intuition.}
We base our theory on a simple premise: \textit{more confident predictions tend to be more accurate}. In other words, if a model assigns a very high probability to the correct answer, its overall entropy will be low.

\begin{assumption}[Entropy-Accuracy Monotonicity]
\label{ass:entropy-accuracy}
For a given input with true label $Y$, suppose model $u$ assigns probability $p_u(Y)$ to $Y$, and $\widetilde{H}_u$ is its normalized entropy. Then, for all models $u, v \in [U]$:
\[
p_u(Y) > p_v(Y) \quad \Rightarrow \quad \widetilde{H}_u < \widetilde{H}_v.
\]
\end{assumption}
While this strict mathematical relationship may not hold perfectly in every single instance, we find that it holds strongly in aggregate practice across all datasets and models tested (see empirical verification in \cref{fig:hvsacc:all} of \cref{app:results:assumption}).

Given this assumption, ETTC simply selects the prediction from the most accurate model for that specific question, denoted as $u^\star$. Let $c^\star := \Pr(\hat{y}_{u^\star} = Y)$ be the accuracy of this best model. ETTC guarantees performance of at least $c^\star$.
Now, to understand where majority voting fails, we must model prediction dependency. Consider a simple coupling scheme: with probability $\lambda$, all the \textit{non-best} models make a correlated error and copy the exact same prediction $W$ (e.g., due to shared biases). With probability $1 - \lambda$, their predictions are independent. Let $\bar{c} := \Pr(W = Y)$ be the accuracy of this correlated ``bloc'' prediction, and let $A_{\mathrm{MV}}(0)$ be the baseline accuracy of majority voting if all models were perfectly independent.

\input{figures/table_entropyvsmv_samesize}

\begin{theorem}[Superiority of ETTC over Voting]
\label{thm:ettc_vs_mv_main}
With the setup above and under \Cref{ass:entropy-accuracy}, let $A_{\min H} := \Pr(\hat y_{\min H} = Y)$ be the accuracy of ETTC. Then for any correlation level $\lambda \in [0,1]$, we have:
\begin{align}
A_{\mathrm{MV}}(\lambda) &= \lambda\,\bar c \;+\; (1 - \lambda)\,A_{\mathrm{MV}}(0), \tag{1}\label{eq:mv_affine} \nonumber\\
A_{\min H} & - A_{\mathrm{MV}}(\lambda) \nonumber
\\ &= \lambda(c^\star - \bar c) \;+\; (1 - \lambda)(A_{\min H} - A_{\mathrm{MV}}(0)). \nonumber%
\end{align}
In particular, $A_{\min H} \ge A_{\mathrm{MV}}(\lambda)$ for all $\lambda$, with strict inequality whenever $\lambda > 0$ and $\bar c < c^\star$.
\end{theorem}

\paragraph{Interpretation.}
We provide the full proof in \cref{app:proof:thm:ettc_vs_mv_main}.
This result mathematically highlights the fundamental flaw of majority voting in ensembles. 
Because voting ignores model quality, it is highly vulnerable to correlated errors. As the error correlation $\lambda$ increases (e.g., multiple weak models sharing the same flaw), voting accuracy degrades and is dragged down toward $\bar c$, which is substantially lower than the expert model's accuracy $c^\star$. 

In contrast, ETTC entirely bypasses this failure mode by selecting the single most confident prediction. 
Under the mild assumption that lower entropy correlates with higher accuracy, ETTC guarantees performance at least as good as the most accurate model, regardless of how many weaker models agree on a wrong answer. Since VLMs heavily share training data and architectures, making their predictions inherently dependent, ETTC offers a structurally safer and more principled aggregation strategy.

\subsection{Empirical Verification}
\label{sec:ettc:results}
We now evaluate ETTC in practical multi-model ensemble settings. We structure our evaluation to answer three key questions: (1) Can ETTC effectively leverage diverse models of similar sizes across different families? (2) Does it remain effective when scaling models within the same architecture family? (3) How robust is the method to extreme capability gaps and modality shifts? 

\subsubsection{Similar-Sized Models from Different Families}
\paragraph{Motivation.}
We first evaluate whether ETTC can better leverage diversity among models of comparable scale but distinct architectural families. In this setting, models offer complementary strengths, but the variance in prediction quality makes standard voting noisy.

\paragraph{Setup.}
We select four models of similar scale (7B-12B): LLaMA-11B, Pixtral-12B, Gemma-12B, and Qwen-7B. These models produce predictions for each dataset, and we compare voting and ETTC on the exact same set of outputs. Notably, no single model consistently dominates across all tasks, and some models are clearly weaker on specific domains, adding noise to the ensemble.

\paragraph{Findings.}
In \cref{tab:entropyvsmv:samesize}, ETTC outperforms voting on five of six datasets, yielding an average accuracy gain of +2.81\% (66.56\% vs.\ 63.75\%). The largest improvements occur on tasks where model performance diverges significantly, such as MathVista and MathVision. In these cases, voting suffers from equal-weighting, allowing the weaker models to dilute the correct signal. 
In contrast, ETTC adaptively prioritizes high-confidence predictions, effectively aligning with the strongest model for each specific item, and often exceeding the best single model's standalone performance.

\paragraph{Takeaway.}
When aggregating diverse but uneven models, ETTC offers a clear advantage over voting: it selectively filters noise from weaker models based on their own uncertainty, making it highly effective in heterogeneous settings.

\input{figures/table_entropyvsmv_samefamily}

\subsubsection{Same-Family Models of Different Scales}
\paragraph{Motivation.}
We then examine whether ETTC remains effective when models share the exact same architecture and training data, but differ vastly in scale. While scaling laws introduce meaningful diversity in capability, the shared inductive biases create high prediction dependency, the exact scenario where our theory suggests voting will struggle.

\paragraph{Setup.}
We use four models from the Qwen family: 3B, 7B, 32B, and 72B. Each model produces predictions on all datasets, and we compare aggregation methods on their combined outputs.

\paragraph{Findings.}
As in \cref{tab:entropyvsmv:samefamily}, ETTC outperforms voting on all datasets, achieving an average gain of +2.84\% (71.68\% vs.\ 68.84\%). While overall prediction correlation is higher than in the cross-family setting, the performance variance introduced by scale provides useful diversity. Specifically, smaller models occasionally make correct predictions with higher certainty than larger ones. 
ETTC successfully detects and leverages these instances, allowing smaller models to override the incorrect predictions of large models. 
Overall, ETTC consistently surpasses the accuracy of the strongest model (Qwen-72B), whereas voting frequently performs worse than the strongest model due to dilution.

\paragraph{Takeaway.}
Despite architectural homogeneity, ensembles of different-sized models still benefit immensely from confidence-based selection. 
ETTC avoids overcounting correlated errors and allows smaller models to meaningfully enhance larger ones, challenging the conventional wisdom that the largest model should dictate test-time performance.

\subsubsection{Robustness and Generalization}
To comprehensively validate the reliability of ETTC, we stress-test it by analyzing the specific risk of dilution in heterogeneous ensembles, extending our evaluation to text-only reasoning, and addressing miscalibration via supervision.

\paragraph{Robustness to Weak Learners.}
In our analysis of same-family models, we combined models with vast capability gaps (e.g., 3B vs.\ 72B). While beneficial for coverage, such ensembles introduce the critical risk of ``dilution'', where weaker models drag down the performance of stronger ones~\citep{nnensemble_nips1994,manybetterthanall_ensemble_ai2002,frugalgpt_ensemble_tmlr2024}.
To investigate whether ETTC can mitigate this risk, we conducted a fine-grained ablation on MathVista evaluating all possible subset combinations of the Qwen family (detailed settings and results in \cref{tab:ensemble_all} of \cref{app:ensemble_robustness}).

We find that voting is highly sensitive to this disparity: when a weak learner is combined with a strong expert, voting accuracy drops significantly, as the noise from the smaller model effectively drowns out the expert's signal.
In contrast, ETTC demonstrates remarkable robustness. In the same weak-strong pairings, ETTC not only avoids degradation but actually surpasses the standalone strong baseline. By relying on predictive entropy, ETTC effectively acts as a filter: it accepts the smaller model's answer only when it is highly confident, while disregarding its uncertain errors. This confirms that ETTC enables ``safe'' ensembling, allowing the integration of disparate models without compromising the expert's performance.

\paragraph{Generalization to Text-Only Reasoning.}
A key question is whether the benefits of ETTC are constrained by the perception bottlenecks inherent to VLMs, or if they reflect a broader property of reasoning models. To address this, we extended our evaluation to the text-only domain using ``Thinking'' LLMs (Qwen-3-Thinking) on standard reasoning benchmarks (ARC-Easy and MMLU-Pro; see \cref{app:thinkingllm}).

Consistent with our VLM findings, ETTC outperforms voting across diverse ensemble configurations in this setting as well. Notably, in highly heterogeneous ensembles (e.g., combining 4B and 235B models), ETTC improves accuracy by nearly 5\% over voting on MMLU-Pro.
This result confirms that the correlation between predictive entropy and correctness is a fundamental property of reasoning models, independent of input modality, and that ETTC acts as an effective noise filter in pure language tasks.

\paragraph{Supervised Calibration.}
Finally, while ETTC is robust in zero-shot settings, it relies on the assumption that model confidence is a reliable proxy for correctness, an assumption that weakens when models are miscalibrated (i.e., ``confidently wrong'').
To mitigate this, we propose a \emph{Supervised ETTC} variant that learns to weigh confidence signals based on empirically observed reliability. We train a lightweight logistic regressor using entropy-based features to predict the likelihood of correctness, allowing the system to dynamically downweight low-entropy predictions from unreliable models (see \cref{app:supervised_ettc}).

Empirically, this variant consistently yields further improvements over unsupervised ETTC, particularly on challenging benchmarks like MathVision~\citep{wang24_nips_mathvision} where base model calibration is weaker. By effectively penalizing overconfidence, the supervised approach achieves the highest overall performance across all settings. This highlights that while entropy is a strong zero-shot signal, minimal supervision can significantly enhance reliability by adapting to specific model failure modes.

\paragraph{Overall Summary.}
Across both ensemble settings (diverse and redundant), ETTC consistently outperforms majority voting without requiring additional training or tuning. These results empirically validate our theoretical findings: when prediction dependency undermines voting, entropy-based selection offers a structurally safer and more adaptive path to test-time improvement.

\section{Related Work}

\paragraph{TTC in LLMs.}
Chain-of-thought (CoT) prompting enables multi-step reasoning~\citep{wei22_nips_cot,kojima22_nips_zsreasoner}, while \emph{self-consistency} (majority voting) improves accuracy by sampling and aggregating diverse reasoning paths~\citep{wang23_iclr_selfconsis}. Recent studies demonstrate that optimally allocating test-time compute (TTC) can sometimes rival scaling up model parameters~\citep{snell24_arxiv_ttc_deepmind}. Advanced TTC methods have also explored self-calibration~\citep{huang2025efficient} and entropy-minimization for test-time adaptation~\citep{zhang2025come}.

\textit{Our Position:} While inference-time scaling is well-established for text, its efficacy in the multimodal domain remains underexplored. Many recent adaptive methods also require training stages or model weight updates. In contrast, our work bridges the modality gap by systematically evaluating pure, inference-only TTC strategies, revealing how visual perception bottlenecks fundamentally alter the effectiveness of test-time scaling.

\paragraph{Enhancing VLM Reasoning.}
To improve VLM reasoning, researchers have adapted visual CoT prompts~\citep{chen24_aaai_vcot} and developed test-time consistency objectives~\citep{chou25_arxiv_ttcvlm, prahitha25_acl_vlmensemble}. Alternatively, post-training methods using Reinforcement Learning from Human Feedback have been proposed to align multimodal reasoning~\citep{sun24_acl_mmrlhf, yu24_cvpr_rlhfv}. 

\textit{Our Position:} Post-training approaches require substantial annotated data and compute budgets, while prompt-specific strategies often lack generalizability. Instead of retraining, we provide a lightweight, inference-only study across diverse, rigorous visual reasoning benchmarks (e.g., MathVista, MMMU). Crucially, we go beyond measuring performance to diagnose \emph{why} standard aggregation strategies fail in VLMs through the lens of prediction dependency.

\paragraph{Ensembles, Uncertainty, and Correlation.}
Classic machine learning theory establishes that ensemble gains depend heavily on prediction \emph{diversity} (i.e., low error correlation) among members~\citep{tumer96_cs_errorensemble,kuncheva03_ml_classifierensemble}. Deep ensembles are widely used to capture predictive uncertainty~\citep{lakshminarayanan17_nips_uncertaintyensemble, guo17_icml_calibrationensemble}, and probabilistic aggregation often relies on confidence-weighted ``opinion pooling''~\citep{rufo_12_ba_llpool,dietrich17_scw_pooling}. 

\textit{Our Position:} While classic literature often pools full probability distributions or requires co-training diverse experts, our Entropy-based TTC method is designed for modern, off-the-shelf generative models. Rather than averaging distributions or relying on simple frequency (voting), we use per-item predictive entropy as a zero-shot filter to dynamically route trust to the most reliable model, enabling smaller models to safely augment larger ones without dilution.

\section{Conclusion}
We presented a systematic investigation into the transferability of TTC strategies from LLMs to VLMs.
Our analysis identifies a critical bottleneck: standard aggregation methods like majority voting are fundamentally limited by the high statistical dependency of VLM predictions.
We theoretically and empirically demonstrate that without diversity, voting offers diminishing returns, and in heterogeneous ensembles, it succumbs to noise from weaker models.
To overcome these limitations, we proposed Entropy-based Test-Time Compute (ETTC), a method that prioritizes prediction confidence over frequency.
ETTC proves to be a robust strategy for leveraging the diversity of multi-model ensembles, enabling significantly smaller models to synergistically enhance larger ones without the risk of dilution.
Furthermore, we showed that this confidence-correctness correlation extends beyond vision, improving reasoning in ``Thinking'' LLMs as well.
Ultimately, our work suggests that the future of efficient test-time scaling lies not just in generating more samples, but in intelligently selecting the most reliable signals from diverse model ecosystems.

\section*{Acknowledgment}
We thank the reviewers for their constructive feedback. 
We also thank Jiaoda Li, Yu Fan, Jingwei Ni, and Chenxi Pang for their valuable input during the early stages of this work.
Yifan Hou is supported by the Swiss Data Science Center PhD Grant (P22-05).

\section*{Impact Statement}

This work aims to improve the reliability and accuracy of Vision-Language Models (VLMs) through inference-time strategies. 
The impact of our proposed method, Entropy-based Test-Time Compute (ETTC), can be analyzed through three primary lenses: reliability, computational efficiency, and potential risks.

\paragraph{Enhancing Reliability in Visual Reasoning.}
By leveraging predictive entropy to filter out uncertain predictions, our method significantly improves the accuracy of VLMs in complex domains, such as mathematics and scientific diagram interpretation. 
This advancement is crucial for deploying VLMs in high-stakes applications (e.g., educational tutoring systems or scientific data analysis) where hallucinated or inconsistent answers can be detrimental. 
By prioritizing high-confidence predictions, our approach helps mitigate the ``stochastic parrot'' phenomenon often observed in multimodal generation, anchoring outputs to more deliberate reasoning paths.

\paragraph{Democratization and Efficient Utilization.}
A key finding of our work is that smaller, less capable models (e.g., 3B or 7B parameters) can synergistically enhance the performance of significantly larger models (e.g., 72B) within an ensemble. 
This has positive implications for the democratization of AI. 
It suggests that users or institutions with limited computational resources can still contribute meaningfully to ensemble systems by deploying smaller models. 
Furthermore, this ``small-helps-large'' dynamic offers a pathway to improve system performance without solely relying on training increasingly massive, energy-intensive models, potentially extending the useful lifespan of existing open-source weights.

\paragraph{Inference Cost and Environmental Impact.}
While our method avoids the massive carbon and financial costs of training new models, Test-Time Compute (TTC) strategies inherently increase inference costs compared to standard greedy decoding (due to multiple sampling rounds or ensembling). 
This leads to higher energy consumption per query. 
However, the ETTC selection mechanism itself introduces negligible computational overhead beyond the initial inference. Moreover, our results suggest that querying a heterogeneous ensemble (e.g., one large model paired with several small, cheap models) using ETTC can yield performance comparable to much larger, more expensive setups. This potentially offers a more favorable trade-off between accuracy and total energy expenditure.

\paragraph{Risks of Over-Reliance on Confidence.}
Our method relies on the assumption that lower entropy (higher confidence) correlates with correctness. 
While we empirically validate this trend, there is a risk that models may be ``confidently wrong,'' particularly in out-of-distribution scenarios or if the base models are poorly calibrated. 
If deployed without safeguards, this could lead users to place unwarranted trust in incorrect outputs simply because the system assigned them a high confidence score. 
To mitigate this, we emphasize that ETTC should be used in conjunction with base models that have undergone rigorous safety alignment and calibration, or paired with supervised variants like the one explored in this work.

\bibliography{example_paper}
\bibliographystyle{icml2026}

\newpage
\appendix
\onecolumn

\section{Theoretical Proofs}
\label{app:proof}
\subsection{Proof of \cref{thm:mv-dependency}}
\label{app:proof:thm:mv-dependency}
\begin{proof}
We provide a theoretical justification for the claim that the improvement from majority voting decreases monotonically with statistical dependency among model predictions. 
We proceed by defining a simple probabilistic coupling model that controls prediction dependency, and then analyze how the expected voting accuracy varies with this dependency level.

\subsubsection{Coupling Model: Copy-or-Independent Sampling}
We assume all $U$ predictions $\{X_u\}_{u=1}^U$ are drawn from a shared coupling mechanism that depends on a parameter $\lambda \in [0, 1]$:
With probability $\lambda$, all predictions are identical copies of a single sample $X$.
With probability $1 - \lambda$, predictions are sampled independently from a shared categorical distribution $\pi = (\pi_1, \dots, \pi_K)$ over $K$ options.
Formally, for any pair $(X_u, X_v)$,
\begin{equation}
\label{eq:copy-iid-coupling}
(X_u, X_v) \sim 
\begin{cases}
(X, X), & \text{with probability } \lambda \\
(X', X''), \quad X', X'' \overset{\text{i.i.d.}}{\sim} \pi, & \text{with probability } 1 - \lambda
\end{cases}
\end{equation}
This ensures uniform pairwise dependency, controlled by $\lambda$.

\subsubsection{Lemma: Behavior of Dependency Metrics under Coupling}
We now show that both statistical dependency metrics used in our main theorem, normalized mutual information and correctness correlation, are monotonic in $\lambda$ under this coupling.

\paragraph{(a) Normalized mutual information.}
Let $X, X'$ be two predictions drawn according to the coupling in \eqref{eq:copy-iid-coupling}. Their joint distribution is
\[
P_\lambda(i,j) = \lambda \cdot \pi_i \cdot \delta_{ij} + (1 - \lambda) \cdot \pi_i \cdot \pi_j,
\]
where $\delta_{ij}$ is the Kronecker delta. The marginal distributions remain unchanged as $\pi$.

Since mutual information $I(X;X')$ increases with $\lambda$ (via the convexity of KL divergence), and the marginals are fixed, the normalized mutual information $\mathrm{NMI}(X; X')$ is also non-decreasing in $\lambda$:
\[
\mathrm{NMI}(X; X') = \frac{I(X; X')}{H(X)} \uparrow \text{ as } \lambda \uparrow.
\]
Hence, the average pairwise NMI $\overline{\mathrm{NMI}}$ is also monotonic in $\lambda$.

\paragraph{(b) Correctness correlation.}
Let $Z_u = \mathbb{I}\{X_u = Y\}$, where $Y$ is the correct option. Denote single-trial accuracy as $p = \mathbb{P}(X_u = Y)$. Then for any pair $(Z_u, Z_v)$:
Under the ``copy'' case: $\mathbb{P}(Z_u = Z_v = 1) = p$.
Under the ``independent'' case: $\mathbb{P}(Z_u = Z_v = 1) = p^2$.

Therefore, the covariance is
\[
\mathrm{Cov}(Z_u, Z_v) = \mathbb{E}[Z_u Z_v] - p^2 = \lambda(p - p^2) = \lambda p(1 - p),
\]
and the correlation is
\begin{equation}
\label{eq:rho-lambda}
\rho(Z_u, Z_v) = \frac{\mathrm{Cov}(Z_u, Z_v)}{p(1 - p)} = \lambda.
\end{equation}
Thus, the average correlation $\overline{\rho} = \lambda$.

\subsubsection{Main Proof: Monotonicity of Voting Improvement}
Let $A_{\mathrm{MV}}(U; \lambda)$ be the expected voting accuracy under dependency level $\lambda$, and let $A_{\mathrm{single}} = p$ be the single-trial accuracy. 

We decompose voting accuracy by conditioning on the latent sampling regime:
\begin{equation}
\label{eq:mv-decomp}
A_{\mathrm{MV}}(U; \lambda) 
= \lambda \cdot A_{\mathrm{MV}}(U;\text{copy}) 
+ (1 - \lambda) \cdot A_{\mathrm{MV}}(U;\text{iid}).
\end{equation}

In the ``copy'' case, all predictions are identical, so voting is equivalent to a single trial: $A_{\mathrm{MV}}(U;\text{copy}) = p$.
In the ``iid'' case, predictions are independent, and voting aggregates $U$ samples from $\pi$; here, accuracy improves with $U$, approaching 1 as $U \to \infty$ if $p > \tfrac{1}{K}$.
Thus:
\begin{align}
A_{\mathrm{MV}}(U; \lambda) 
&= \lambda p + (1 - \lambda) A_{\mathrm{MV}}(U; 0), \\
\Delta A_{\mathrm{MV}}(U; \lambda) 
&:= A_{\mathrm{MV}}(U; \lambda) - p = (1 - \lambda)(A_{\mathrm{MV}}(U; 0) - p).
\label{eq:mv-improvement}
\end{align}
The improvement $\Delta A_{\mathrm{MV}}(U; \lambda)$ is thus a linear function decreasing in $\lambda$, and since $\lambda = \overline{\rho}$ (from \eqref{eq:rho-lambda}) and $\overline{\mathrm{NMI}}$ increases with $\lambda$, voting improvement is monotonically decreasing in both.

\subsubsection{Corollary (Extremes)}
If $\lambda = 1$ (i.e., $\overline{\rho} = 1$ or $\overline{\mathrm{NMI}} = 1$), then all predictions are identical and voting offers no improvement:
  \[
  \Delta A_{\mathrm{MV}}(U) = 0.
  \]
If $\lambda = 0$ (i.e., predictions are independent) and $p > \tfrac{1}{K}$, then:
  \[
  A_{\mathrm{MV}}(U) \to 1 \quad \text{as} \quad U \to \infty.
  \]

\end{proof}

\subsubsection{Discussion}
This result formalizes an intuitive principle: confidence-based aggregation (e.g., voting) helps only when predictions are sufficiently diverse. High dependency, measured either via correctness correlation or mutual information, reduces the effective information gain from additional samples. Empirical results confirm this trend across VLMs and datasets: voting yields larger gains when dependency is low.

\subsection{Proof of \Cref{thm:ettc_vs_mv_main}}
\label{app:proof:thm:ettc_vs_mv_main}

\begin{proof}
\textbf{Setup.}
Fix a $K$-way classification item with true label $Y$. Let $u^\star := \arg\max_u p_u(Y)$ be the best model and define $c^\star := \Pr(\hat y_{u^\star} = Y)$. Let $\mathcal{B} = \{u \ne u^\star\}$ be the set of non-best models, with $|\mathcal{B}| \ge 2$.

\textbf{Coupling among non-best models.}
Introduce a latent variable $L \in \{\text{copy}, \text{iid}\}$:
- With probability $\lambda$, $L = \text{copy}$ and all non-best models predict a shared label $W$; define $\bar c := \Pr(W = Y)$.
- With probability $1 - \lambda$, $L = \text{iid}$ and the non-best predictions are drawn independently.

\textbf{Step 1: Accuracy of ETTC.}
Under Assumption~\ref{ass:entropy-accuracy}, ETTC selects $\hat y_{u^\star}$, so:
\begin{equation}
A_{\min H} = \Pr(\hat y_{u^\star} = Y) = c^\star. \label{eq:ettc_acc}
\end{equation}

\textbf{Step 2: Accuracy of Voting.}
By law of total probability:
\begin{align}
A_{\mathrm{MV}}(\lambda) &= \lambda\, \Pr(\widehat{Y}_{\mathrm{MV}} = Y \mid L = \text{copy}) + (1-\lambda)\, A_{\mathrm{MV}}(0). \label{eq:mv_total_prob}
\end{align}
Under $L = \text{copy}$, all non-best models predict $W$, forming a majority:
\begin{equation}
\Pr(\widehat{Y}_{\mathrm{MV}} = Y \mid L = \text{copy}) = \Pr(W = Y) = \bar c.
\end{equation}
Plugging into \eqref{eq:mv_total_prob}, we recover:
\begin{equation}
A_{\mathrm{MV}}(\lambda) = \lambda\, \bar c + (1 - \lambda)\, A_{\mathrm{MV}}(0). \label{eq:mv_affine_app}
\end{equation}

\textbf{Step 3: Difference and monotonicity.}
Subtracting \eqref{eq:mv_affine_app} from \eqref{eq:ettc_acc}:
\begin{align}
A_{\min H} - A_{\mathrm{MV}}(\lambda)
= \lambda (c^\star - \bar c) + (1 - \lambda)(c^\star - A_{\mathrm{MV}}(0)). \label{eq:gap_affine_app}
\end{align}
This gap is nondecreasing in $\lambda$:
\[
\frac{d}{d\lambda}(A_{\min H} - A_{\mathrm{MV}}(\lambda)) = A_{\mathrm{MV}}(0) - \bar c \ge 0.
\]

\textbf{Step 4: Dominance threshold.}
Let
\[
\lambda^\star = \max\left\{0, \frac{A_{\mathrm{MV}}(0) - c^\star}{A_{\mathrm{MV}}(0) - \bar c} \right\}.
\]
Then for all $\lambda \ge \lambda^\star$, ETTC outperforms voting; if $\bar c < c^\star$ and $\lambda > \lambda^\star$, the gap is strict.
\end{proof}

\noindent\textbf{Remarks.}
- Since $u^\star$ is the best model, typically $\bar c < c^\star$ unless all models perform equally well.
- If $A_{\mathrm{MV}}(0) \le c^\star$, then $\lambda^\star = 0$: ETTC dominates voting at all dependency levels.
- Under the copy-or-independent model, the average correctness correlation among non-best models equals $\lambda$ (see \cref{app:proof:thm:mv-dependency}), providing a direct link between dependency and the TTC advantage.

\section{Experiment Settings}

\subsection{Dataset}\label{app:setting:dataset}
\input{figures/table_dataset_details}
We evaluate our methods on six diverse multi-choice benchmarks spanning three domains: mathematical reasoning (MathVista, MathVision), diagram-based QA (TQA, ScienceQA), and general visual understanding (MMStar, MMMU).
\Cref{tab:dataset_stats} summarizes key statistics, including dataset size, official split used, and number of answer options. Note that some datasets contain variable numbers of options (e.g., 2 - 9 in MMMU), which adds to the challenge and makes majority voting less stable. This diversity ensures our evaluation reflects a wide range of real-world reasoning settings.

\subsection{Prompt}\label{app:setting:prompt}
\input{figures/appendix/prompt_ncot}
\input{figures/appendix/prompt_cot}
To ensure consistency and minimize response variance across models, we standardize the prompting format in all benchmark evaluations. Specifically, we use a direct QA prompt without explanation, and a chain-of-thought (CoT) style prompt when evaluating reasoning performance or conducting consistency analysis. Below, we show two representative examples for comparison. The image and question are kept identical, while only the prompt template changes.

\subsection{Baselines}\label{app:baselines}
To better assess the reliability of CoT responses, we include several shallow feature-based baselines. These models predict the correctness of a response using surface-level properties, without access to model internals or gradient signals.

\paragraph{Pivot words.}

\input{figures/appendix/pivot_phrases}

Pivot words are rhetorical expressions that signal shifts in reasoning, such as realization, verification, or synthesis. Prior work~\citep{lippmann25_arxiv_pivotword_substance} suggests that the presence of such expressions often correlates with more deliberate and structured reasoning.
We use a curated list of phrases categorized by rhetorical function, shown in \Cref{tab:pivot}. These are used as features for correctness prediction (e.g., counting their presence in CoTs).

\paragraph{Vague words.}

\input{figures/appendix/vague_phrases}

Vague expressions are often used to hedge or express uncertainty, and may correlate with lower confidence or correctness in model reasoning. We group these into two categories, uncertainty and hedging—based on their rhetorical function. See \Cref{tab:vague}.

\paragraph{Feature-All.}

\input{figures/appendix/feature_all}

We also define a feature set that combines lexical and stylistic signals for each CoT response. Specifically, we consider four interpretable features: response length (token count), lexical diversity (unique token count), number of pivot words, and number of vague words. See \Cref{tab:featureall} for detailed definitions. For prediction, we compute the sum of these feature values for each example, encouraging longer, more expressive, and more structured responses, while penalizing vague expressions. The model response with the highest total score is selected as the final prediction.

\section{Supplementary Results}

\subsection{Voting Improvement vs. $\overline{\mathrm{NMI}}$ and Correlation}\label{app:results:mv_correlation}
\input{figures/improvement_nmi_all}

\input{figures/improvement_corr_all}

While the overall trends in \Cref{fig:improvement_nmi:all,fig:improvement_corr:all} are consistent with our theoretical expectations, MathVision stands out as an exception. Specifically, we observe weaker or even inverted correlation between prediction dependency and voting improvement on this dataset. A likely explanation is that MathVision poses significantly higher difficulty for current VLMs, its average accuracy across models is around 30\%, which suggests that models are often uncertain or guessing. In such low-performance regimes, prediction behaviors may become erratic or overly stochastic, reducing the reliability of entropy, correlation, and voting-based signals. As a result, the dependency measures may not reflect meaningful error structure, making voting behavior less predictable.

\subsection{Empirical Evidence to Support Assumption}\label{app:results:assumption}
\input{figures/assumption_all}
\Cref{fig:hvsacc:all} shows the relationship between normalized entropy $\widetilde{H}_u$ and accuracy across multiple models on six benchmarks. We observe a strong inverse correlation between entropy and accuracy, consistent with our Entropy-Accuracy Monotonicity assumption (\Cref{ass:entropy-accuracy}). Higher-performing models generally exhibit lower entropy, indicating more confident and reliable predictions.

\subsection{Ensemble Robustness Analysis}
\label{app:ensemble_robustness}
\input{figures/appendix/ensemble_all}

\paragraph{Setup.}
To investigate whether including weaker, smaller models in an ensemble degrades performance, we conducted a comprehensive ablation study using the Qwen-2.5-VL family (3B, 7B, 32B, 72B) on the MathVista benchmark. We evaluated all possible pairwise and triplet combinations to simulate diverse ensemble qualities (\cref{tab:ensemble_all}).

\paragraph{Findings.}
The results highlight a critical failure mode of majority voting in heterogeneous ensembles. For instance, when combining the weakest model (3B, $\sim$52\% accuracy) with the strongest (72B, $\sim$80\% accuracy), voting performance drops significantly to 73.15\%, effectively dragging the strong model down toward the average. This confirms that voting is vulnerable when the ensemble contains models with large capability gaps. 

In contrast, \textit{ETTC demonstrates remarkable robustness}. In the same 3B+72B setting, ETTC achieves 84.81\%, not only avoiding the degradation seen in voting but actually surpassing the standalone performance of the 72B model by over 4\%. This trend holds across triplet configurations as well; for example, in the \{3B, 7B, 72B\} ensemble, voting achieves 82.22\% while ETTC reaches 84.44\%.

\paragraph{Takeaway.}
These findings demonstrate that ETTC effectively utilizes predictive entropy to ``filter'' unreliable signals from weaker models while still leveraging their occasional correct, high-confidence predictions. Unlike voting, which requires careful curation of similarly-capable models to avoid dilution, ETTC allows for \textit{safe} ensembling: users can integrate smaller, cheaper models (like the 3B) to boost larger ones without the risk of degrading overall performance.

\subsection{Generalization to Thinking LLMs}
\label{app:thinkingllm}
\input{figures/appendix/results_thinking}
\paragraph{Setup.}
To verify that our findings are not an artifact of the visual modality or specific to Vision-Language Models (VLMs), we extended our evaluation to text-only reasoning tasks using Thinking LLMs. We employed the Qwen-3-Thinking family (4B, 30B, and 235B parameters) and evaluated them on two established reasoning benchmarks: ARC-Easy (common sense reasoning) and MMLU-Pro (mathematics subset). We tested various ensemble configurations, including combinations of models with vast size discrepancies (e.g., 4B + 235B).

\paragraph{Findings.}
As shown in \cref{tab:thinking_llm}, the benefits of ETTC generalize robustly to the text domain. Across all 8 ensemble configurations, ETTC consistently outperforms majority voting. Notably, on the MMLU-Pro dataset, aggregating the 4B and 30B models with voting yields 89.34\%, significantly underperforming the standalone 30B model (94.12\%) due to the noise introduced by the smaller model. In contrast, ETTC achieves 94.08\%, effectively recovering the performance of the strong model by filtering out the 4B model's low-confidence errors. Furthermore, in the most heterogeneous ensemble (4B+235B), ETTC improves upon voting by nearly 5 points on MMLU-Pro (94.67\% vs 89.79\%), demonstrating its ability to safely leverage small-model compute without diluting the quality of large-model outputs.

\paragraph{Takeaway.}
These results confirm that the correlation between predictive entropy and correctness is a fundamental property of reasoning models, whether multimodal or text-only. ETTC's success with ``Thinking'' models suggests it is a general-purpose, modality-agnostic strategy for enhancing test-time reliability in heterogeneous ensembles.

\subsection{Supervised ETTC}
\label{app:supervised_ettc}

We provide additional details on the supervised variant of ETTC, which learns from a small set of labeled question--model pairs when low entropy is a \emph{reliable} signal of correctness.

\paragraph{Problem setting.}
Given $Q$ questions and $M$ models, each model $u$ produces a predictive distribution $p_{qu}(\cdot)$ over $K$ options for question $q$, aggregated over $U{=}16$ stochastic decoding samples (see \cref{sec:when}). The goal is to learn a function that predicts whether a model's low-entropy output is likely to be correct.

\paragraph{Feature construction.}
For each $(q,u)$ pair, we compute two features:
\[
\widetilde{H}_{qu} := -\tfrac{1}{\log K} \sum_{k=1}^K p_{qu}(k)\log p_{qu}(k), 
\quad
\mathrm{RelEnt}_{qu} := \frac{\widetilde{H}_{qu}-\min_v \widetilde{H}_{qv}}{\max_v \widetilde{H}_{qv}-\min_v \widetilde{H}_{qv}}.
\]
Here $\widetilde{H}_{qu}$ is the normalized entropy of model $u$, while $\mathrm{RelEnt}_{qu}$ contextualizes this entropy relative to other models for the same question. The final feature vector is $(\widetilde{H}_{qu}, \mathrm{RelEnt}_{qu}) \in \mathbb{R}^2$.

\paragraph{Labels and classifier.}
The binary label is
\[
Z_{qu} := \mathbb{I}\{\hat y_{qu} = Y_q\},
\]
where $\hat y_{qu}$ is the top-1 prediction and $Y_q$ is the ground truth. We train a logistic regression classifier to predict $\Pr(Z_{qu}=1)$ from the entropy features.

\paragraph{Training protocol.}
To simulate low-resource conditions, we use two-fold cross-validation across questions: each dataset is split into halves, one for training and one for testing, with roles reversed in a second run. This prevents test leakage and mimics scenarios where only limited annotations are available.

\paragraph{Inference rule.}
At test time, for each $(q,u)$ we compute the adjusted score
\[
\mathrm{Score}_{qu} := \widetilde{H}_{qu} \cdot (1 - \hat{p}_{qu}),
\]
where $\hat{p}_{qu}$ is the predicted correctness probability from the classifier. We then select the model with the lowest score:
\[
u_q^\star := \arg\min_u \mathrm{Score}_{qu},
\quad
\widehat{Y}_q := \hat y_{q u_q^\star}.
\]
This rule penalizes overconfident but unreliable predictions while rewarding trustworthy ones.

\input{figures/results_supervised.tex}

\paragraph{Results.}
As in \cref{tab:entropyvsmv:merged}, supervised ETTC outperforms both voting and unsupervised ETTC across datasets and ensemble settings. Gains are largest on ambiguous tasks (e.g., MathVision, MMStar, MMMU), where entropy alone is less reliable. Even with only two-fold cross-fitting and no extra supervision, the classifier learns to identify failure modes of entropy selection, making more robust choices and underlining the value of combining entropy with supervised error modeling.

\section*{Limitations}
Our study focuses on multiple-choice visual reasoning tasks and assumes access to model confidence scores via output distributions. The proposed methods, especially entropy-based selection, may not directly generalize to open-ended tasks or models lacking probabilistic outputs. 
Additionally, while our evaluation covers diverse datasets and model ensembles, the gains of supervised entropy-based TTC depend on the quality and availability of annotated examples, which may be costly to obtain in some domains. 
Lastly, our analysis assumes that entropy correlates with accuracy, which may not hold for all models or tasks.

\section*{LLM Usage}
We used ChatGPT as general-purpose assistive tools during the preparation of this paper. 
Specifically, LLMs were employed for polishing grammar, improving clarity, formatting LaTeX, generating illustrative figures, and debugging minor code snippets. 
LLMs were not involved in research ideation, experimental design, or the development of theoretical results.

\end{document}

%% file: figures/results_ttc.tex
\begin{figure*}[!t]
\vspace{.2cm}
    \centering
    \begin{subfigure}{.37\textwidth}
    \centering
    \scalefont{0.8}
    \begin{tikzpicture}
    \begin{axis}[
        ybar,
        bar width=0.1cm,
        width=6cm, height=4.cm,
        enlarge x limits=0.1,
        ylabel={Accuracy (\%)},
        xlabel near ticks,
        ylabel near ticks,
        x tick label style={rotate=30, anchor=east},
        ymin=20, ymax=85,
        xtick=data,
        xticklabel style={font=\scriptsize},
        symbolic x coords={\textbf{Overall}, MathVista, MathVision, TQA, ScienceQA, MMStar, MMMU},
        legend style={at={(0.5,1.1)}, anchor=south, legend columns=5, nodes={scale=0.8, transform shape}},
        ymajorgrids=true,
        grid style=dashed,
    ]

    \addplot+[line width=0.3mm, color=ETHRed, postaction={
        pattern=north east lines
    }] coordinates { 
    (\textbf{Overall}, 57.191877551758616)
    (MathVista, 63.18783068783068)
    (MathVision, 28.707222118612457)
    (TQA, 75.20493585562079)
    (ScienceQA, 79.41913379134499)
    (MMStar, 48.23928571428572)
    (MMMU, 48.39285714285714)
    };
    \addlegendentry{Vanilla}
    \addplot+[line width=0.3mm, color=ETHPurple] coordinates { 
    (\textbf{Overall}, 58.14052218529041)
    (MathVista, 64.32870370370371)
    (MathVision, 29.332571801566576)
    (TQA, 76.44379212872364)
    (ScienceQA, 80.48949996458671)
    (MMStar, 48.95952380952381)
    (MMMU, 49.28904170363797)
    };
    \addlegendentry{Majority Voting}

    \end{axis}
    \end{tikzpicture}
    \caption{\textbf{Direct Answer Prompt.}}
    \label{fig:results_ncot}
    \end{subfigure}
    \begin{subfigure}{.62\textwidth}
    \centering
    \scalefont{0.8}
    \begin{tikzpicture}
    \begin{axis}[
        ybar,
        bar width=0.1cm,
        width=11cm, height=4.cm,
        enlarge x limits=0.1,
        xlabel near ticks,
        ylabel near ticks,
        x tick label style={rotate=30, anchor=east},
        ymin=20, ymax=85,
        xtick=data,
        xticklabel style={font=\scriptsize},
        symbolic x coords={\textbf{Overall}, MathVista, MathVision, TQA, ScienceQA, MMStar, MMMU},
        legend style={at={(0.5,1.1)}, anchor=south, legend columns=5, nodes={scale=0.8, transform shape}},
        ymajorgrids=true,
        grid style=dashed,
    ]

    \addplot+[line width=0.3mm, color=ETHRed, postaction={
        pattern=north east lines
    }] coordinates { 
    (\textbf{Overall}, 58.918526253062495)
    (MathVista, 65.18353174603175)
    (MathVision, 30.65495617306975)
    (TQA, 76.2203739943466)
    (ScienceQA, 78.46740208230044)
    (MMStar, 52.389285714285705)
    (MMMU, 50.59560780834074)
    };
    \addlegendentry{Vanilla}
    \addplot+[line width=0.3mm, color=ETHBronze] coordinates { 
    (\textbf{Overall}, 58.52597677289973)
    (MathVista, 64.8941798941799)
    (MathVision, 29.839612085042894)
    (TQA, 76.03392041748207)
    (ScienceQA, 78.26333309724485)
    (MMStar, 51.61904761904762)
    (MMMU, 50.50576752440106)
    };
    \addlegendentry{Pivot Word}
    \addplot+[line width=0.3mm, color=ETHGreen] coordinates { 
    (\textbf{Overall}, 56.622798379203196)
    (MathVista, 61.79894179894181)
    (MathVision, 27.72286460276016)
    (TQA, 73.58121330724069)
    (ScienceQA, 77.06636447340463)
    (MMStar, 50.16190476190475)
    (MMMU, 49.40550133096717)
    };
    \addlegendentry{CoT Length}
    \addplot+[line width=0.3mm, color=ETHBlue] coordinates { 
    (\textbf{Overall}, 59.11383228727514)
(MathVista, 65.34391534391534)
(MathVision, 30.790749720253636)
(TQA, 76.33398564905413)
(ScienceQA, 78.18542389687654)
(MMStar, 52.142857142857146)
(MMMU, 50.29281277728482)
    };
    \addlegendentry{Feature-All}
    \addplot+[line width=0.3mm, color=ETHPurple] coordinates { 
    (\textbf{Overall}, 62.007564105404846)
    (MathVista, 69.46263227513228)
    (MathVision, 32.95295598657218)
    (TQA, 79.68063709502066)
    (ScienceQA, 81.110117572066)
    (MMStar, 55.26428571428571)
    (MMMU, 53.57475598935226)
    };
    \addlegendentry{Majority Voting}
    \end{axis}
    \end{tikzpicture}
    \caption{\textbf{CoT Prompt.}}
    \label{fig:results_cot}
    \end{subfigure}
    \caption{Comparison of test-time compute (TTC) strategies under two prompting styles. 
    In \textbf{Direct Answer} (left), models are instructed to output only the final answer without reasoning; feature-based methods are inapplicable, and majority voting shows no improvement. 
    In \textbf{CoT} (right), models are prompted to reason step by step. While feature-based methods yield no gains, voting offers modest but consistent improvement across datasets.}
    \label{fig:results_ttc}
\end{figure*}
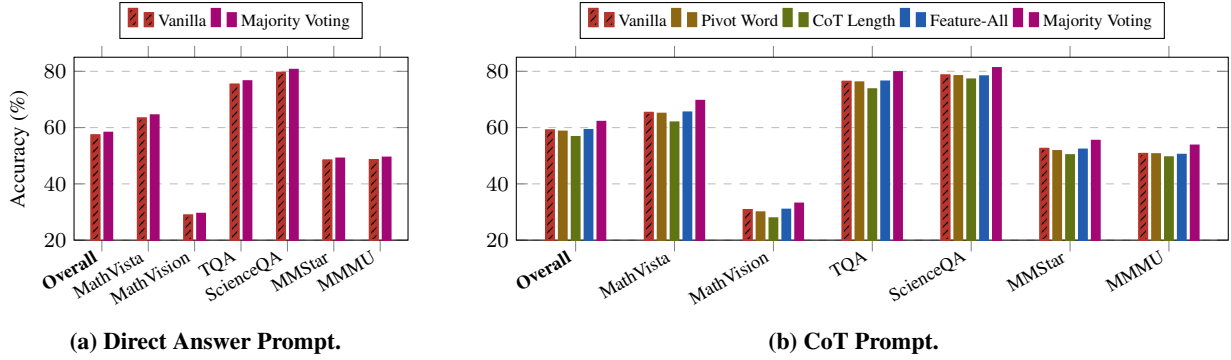

%% file: figures/convergence_nimcorr.tex
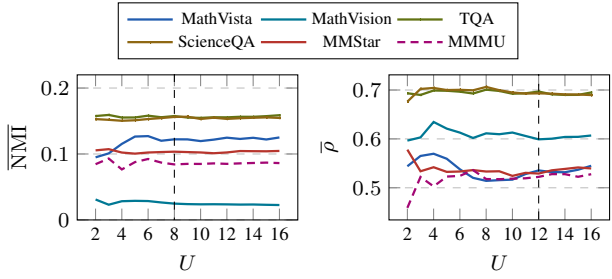
\begin{figure}[!t]
\vspace{1.1cm}
    \centering
    \begin{subfigure}[t]{.49\linewidth}
        \centering
        \scalefont{0.8}
        \begin{tikzpicture}
        \begin{axis}[
            width=4.5cm, height=3.5cm,
            ylabel={$\overline{\mathrm{NMI}}$},
            xlabel={$U$},
            xlabel near ticks,
            ylabel near ticks,
            xtick={2,4,6,8,10,12,14,16},
            xticklabel style={font=\scriptsize},
            symbolic x coords={2,3,4,5,6,7,8,9,10,11,12,13,14,15,16},
            legend style={at={(2.,1.5)}, legend columns=3, nodes={scale=0.8, transform shape},/tikz/overlay,},
            ymajorgrids=true,
            grid style=dashed,
            scaled y ticks=false,
            yticklabel style={/pgf/number format/fixed, /pgf/number format/precision=2},
            ytick={0, 0.1, 0.2, 0.3},
            ymin=0, ymax=0.22 %
            ]
        \addplot+[line width=0.3mm, mark size=0pt, color=ETHBlue] coordinates {
            (2, 0.09493800391344606)
            (3, 0.10087608086261894)
            (4, 0.11530239464656945)
            (5, 0.12641969130949288)
            (6, 0.12710150648682958)
            (7, 0.1200567590176953)
            (8, 0.12225441107304147)
            (9, 0.12213936578729277)
            (10, 0.11955810615889854)
            (11, 0.12176625983051008)
            (12, 0.12472948060290967)
            (13, 0.12282239529794765)
            (14, 0.12455954233191063)
            (15, 0.1219981838659735)
            (16, 0.12506253773210915)

        };
        \addlegendentry{MathVista}
        \addplot+[line width=0.3mm, mark size=0pt, color=ETHPetrol] coordinates {
            (2, 0.030986616562334254)
            (3, 0.023077307659705778)
            (4, 0.028386502907469857)
            (5, 0.028944177164472052)
            (6, 0.028626243604297146)
            (7, 0.026493066221874084)
            (8, 0.024742940421719956)
            (9, 0.023974717253262282)
            (10, 0.023671200642108647)
            (11, 0.023794556736672346)
            (12, 0.023599302127468077)
            (13, 0.02326781029194378)
            (14, 0.02340685734226509)
            (15, 0.022939804015579164)
            (16, 0.02259082046024646)

        };
        \addlegendentry{MathVision}
        \addplot+[line width=0.3mm, mark size=0pt, color=ETHGreen] coordinates {
            (2, 0.1575283020033133)
            (3, 0.15919870650932363)
            (4, 0.15528762617665495)
            (5, 0.1553725703267822)
            (6, 0.1580895796543632)
            (7, 0.15571114441588127)
            (8, 0.1575095116818985)
            (9, 0.15528167076270014)
            (10, 0.1547330718645183)
            (11, 0.15569623943579627)
            (12, 0.1555654527630176)
            (13, 0.15641719396269782)
            (14, 0.1563811578761546)
            (15, 0.15737540086887442)
            (16, 0.15879199656075815)

        };
        \addlegendentry{TQA}
        \addplot+[line width=0.3mm, mark size=0pt, color=ETHBronze] coordinates {
            (2, 0.15289778693947867)
            (3, 0.15179961464917915)
            (4, 0.15036177231730175)
            (5, 0.15111682067581214)
            (6, 0.15299103483413967)
            (7, 0.1544822266233223)
            (8, 0.15618955504407644)
            (9, 0.1564803184025755)
            (10, 0.15308931894698394)
            (11, 0.1550102020941444)
            (12, 0.15317899556488676)
            (13, 0.1542530147675887)
            (14, 0.15490697493302985)
            (15, 0.15549445821199315)
            (16, 0.15466061033373482)

        };
        \addlegendentry{ScienceQA}
        \addplot+[line width=0.3mm, mark size=0pt, color=ETHRed] coordinates {
            (2, 0.10548110807406717)
            (3, 0.10731528417808663)
            (4, 0.10219089698934224)
            (5, 0.1005199312064893)
            (6, 0.10224603873603003)
            (7, 0.10293308335154061)
            (8, 0.10349213229620734)
            (9, 0.10287319025855617)
            (10, 0.10230871051189525)
            (11, 0.10112500046317525)
            (12, 0.10232824229359573)
            (13, 0.1044661718006984)
            (14, 0.10423695004325846)
            (15, 0.1040816383728089)
            (16, 0.10461348726527127)

        };
        \addlegendentry{MMStar}
        \addplot+[line width=0.3mm, mark size=0pt, color=ETHPurple] coordinates {
            (2, 0.08451058037539709)
            (3, 0.09309013942974187)
            (4, 0.07658432865993088)
            (5, 0.08767967183950906)
            (6, 0.092431638087479)
            (7, 0.08697682164250706)
            (8, 0.0841765605668457)
            (9, 0.08501299981719702)
            (10, 0.08490624723916079)
            (11, 0.0857473892520195)
            (12, 0.08501500248052185)
            (13, 0.08578243699193165)
            (14, 0.08633878681419789)
            (15, 0.08690589444256495)
            (16, 0.0861867362694768)

        };
        \addlegendentry{MMMU}
        \draw [dashed] (8,-1) -- (8,1);
        \end{axis}
        \end{tikzpicture}
        \label{fig:convergence_nmi:qwen7b}
    \end{subfigure}
    \begin{subfigure}[t]{.49\linewidth}
        \centering
        \scalefont{0.8}
        \begin{tikzpicture}
        \begin{axis}[
            width=4.5cm, height=3.5cm,
            ylabel={$\overline{\rho}$},
            xlabel={$U$},
            xlabel near ticks,
            ylabel near ticks,
            xtick={2,4,6,8,10,12,14,16},
            xticklabel style={font=\scriptsize},
            symbolic x coords={2,3,4,5,6,7,8,9,10,11,12,13,14,15,16},
            legend style={at={(0.5,1.15)}, anchor=south, legend columns=7, nodes={scale=0.8, transform shape}},
            ymajorgrids=true,
            grid style=dashed,
        ]
        \addplot+[line width=0.3mm, mark size=0pt, color=ETHBlue] coordinates {
            (2, 0.5440036479708164)
            (3, 0.5650933410714533)
            (4, 0.5696261171544478)
            (5, 0.5593865031679065)
            (6, 0.5378039327677089)
            (7, 0.5203519075408979)
            (8, 0.514172364294791)
            (9, 0.5156380651847571)
            (10, 0.5166457954473842)
            (11, 0.5280918919764648)
            (12, 0.5355104354599595)
            (13, 0.5321047322045217)
            (14, 0.5318102491184277)
            (15, 0.5369050489096862)
            (16, 0.5449761036644336)

        };
        \addplot+[line width=0.3mm, mark size=0pt, color=ETHPetrol] coordinates {
            (2, 0.5965565097397895)
            (3, 0.6030195285945913)
            (4, 0.634602726038166)
            (5, 0.6209929880901437)
            (6, 0.612863124347124)
            (7, 0.6018821134070151)
            (8, 0.611376797716822)
            (9, 0.6095603711943225)
            (10, 0.6129736860507379)
            (11, 0.6059499630776096)
            (12, 0.5991177251899719)
            (13, 0.6004296036292074)
            (14, 0.6039613571309618)
            (15, 0.6041794152952277)
            (16, 0.6069284023788312)

        };
        \addplot+[line width=0.3mm, mark size=0pt, color=ETHGreen] coordinates {
            (2, 0.6932983013564707)
            (3, 0.6898824964573743)
            (4, 0.6989832383738386)
            (5, 0.6984600399546534)
            (6, 0.6965670821255953)
            (7, 0.6928511791659514)
            (8, 0.7007112137017196)
            (9, 0.6980635425599911)
            (10, 0.6923952313016724)
            (11, 0.6932688496284741)
            (12, 0.6968550679514729)
            (13, 0.6911117800161924)
            (14, 0.6899668032379777)
            (15, 0.6900582367185785)
            (16, 0.6949457797515505)

        };
        \addplot+[line width=0.3mm, mark size=0pt, color=ETHBronze] coordinates {
            (2, 0.6761803348514138)
            (3, 0.702160461323051)
            (4, 0.7041821829771363)
            (5, 0.6999198093463934)
            (6, 0.7003665228685988)
            (7, 0.6993886979705815)
            (8, 0.7062794009300036)
            (9, 0.6996196797219006)
            (10, 0.6949260870528252)
            (11, 0.6926830279020572)
            (12, 0.6931796528658999)
            (13, 0.6925330344925046)
            (14, 0.6910334792768094)
            (15, 0.6907896301914794)
            (16, 0.6896679246607068)

        };
        \addplot+[line width=0.3mm, mark size=0pt, color=ETHRed] coordinates {
            (2, 0.5779808959156785)
            (3, 0.5337993859413297)
            (4, 0.5420240911062996)
            (5, 0.5325435899251258)
            (6, 0.5333025492601113)
            (7, 0.5361843824666054)
            (8, 0.5332775062278469)
            (9, 0.5338017155306315)
            (10, 0.5243220330720987)
            (11, 0.5305536017291262)
            (12, 0.529102746261505)
            (13, 0.5357470730094044)
            (14, 0.5385715845555661)
            (15, 0.541592808792794)
            (16, 0.5395095389761869)

        };
        \addplot+[line width=0.3mm, mark size=0pt, color=ETHPurple] coordinates {
            (2, 0.45897971124499903)
            (3, 0.5221119582330601)
            (4, 0.5028426729485123)
            (5, 0.5229385667912453)
            (6, 0.5243309594321055)
            (7, 0.5361485882361537)
            (8, 0.5178759501839375)
            (9, 0.5175947158417443)
            (10, 0.5186469793354184)
            (11, 0.519436637316106)
            (12, 0.5223168442251747)
            (13, 0.5279432278940642)
            (14, 0.5278480466209741)
            (15, 0.5224872986025966)
            (16, 0.5277817141225475)

        };
        \draw [dashed] (12,-1) -- (12,1);
        \end{axis}
        \end{tikzpicture}
        \label{fig:convergence_corr:qwen7b}
    \end{subfigure}
    \caption{Convergence of dependency with decoding sample size $U$ on Qwen-7B. Both $\overline{\mathrm{NMI}}$ and $\overline{\rho}$ stabilize when $U{=}12$, suggesting that a moderate number of samples is sufficient to estimate dependency reliably.}
    \label{fig:convergence_nimcorr:qwen7b}
\end{figure}

%% file: figures/improvement_nmicorr.tex
\begin{figure}[!t]
    \centering
    \vspace{.7cm}
    \begin{subfigure}[t]{.48\linewidth}
        \centering
        \scalefont{0.8}
        \begin{tikzpicture}
        \begin{axis}[
          width=\linewidth,
          width=4.5cm, height=3.5cm,
          xmin=0, xmax=0.4,
          ymin=0.01, ymax=0.06,
          xlabel={$\overline{\mathrm{NMI}}$},
          ylabel={$\Delta A_{\mathrm{MV}}(16)$},
          xticklabel style={font=\scriptsize},
          legend style={at={(2.522,1.45)}, legend columns=7, nodes={scale=0.8, transform shape},/tikz/overlay,},
          ymajorgrids=true, 
          grid style=dashed,
        ]
        
\addplot[
          only marks, mark=o, mark size=3.2pt,
          draw=black, fill=ETHGray, line width=0.6pt
        ] coordinates {(0.10865103143693278, 0.04052922210013088)};
        \addlegendentry{LLaMA}
        
        \addplot[
          only marks, mark=square*, mark size=3.2pt,
          draw=black, fill=ETHPurple, line width=0.6pt
        ] coordinates {(0.16694106474118295, 0.03134322928026256)};
        \addlegendentry{Gemma}
        
        \addplot[
          only marks, mark=triangle*, mark size=3.6pt,
          draw=black, fill=ETHRed, line width=0.6pt
        ] coordinates {(0.2960765573340331, 0.018896999552622917)};
        \addlegendentry{Pixtral}
        
        \addplot[
          only marks, mark=diamond*, mark size=3.6pt,
          draw=black, fill=ETHBronze, line width=0.6pt
        ] coordinates {(0.058333002760592056, 0.05535514537067254)};
        \addlegendentry{Qwen-3B}
        
        \addplot[
          only marks, mark=otimes*, mark size=3.6pt,
          draw=black, fill=ETHGreen, line width=0.6pt
        ] coordinates {(0.20269462097893953, 0.02934642553082467)};
        \addlegendentry{Qwen-7B}
        
        \addplot[
          only marks, mark=star, mark size=2.6pt,
          draw=black, fill=ETHPetrol, line width=0.6pt
        ] coordinates {(0.314257221899171, 0.021071888203437306)};
        \addlegendentry{Qwen-32B}
        
        \addplot[
          only marks, mark=pentagon*, mark size=3.6pt,
          draw=black, fill=ETHBlue, line width=0.6pt
        ] coordinates {(0.33757829526989575, 0.020040048892834172)};
        \addlegendentry{Qwen-72B}
        
        \pgfplotstableread{
        x   y
0.10865103143693278 0.04052922210013088
0.16694106474118298 0.03134322928026256
0.29607655733403315 0.018896999552622917
0.058333002760592056 0.05535514537067254
0.20269462097893953 0.02934642553082467
0.314257221899171 0.021071888203437306
0.33757829526989575 0.020040048892834172
        }\allpoints
        
        \pgfplotstablecreatecol[linear regression={y=y,x=x}]{regression}{\allpoints}
        \pgfmathsetmacro{\m}{\pgfplotstableregressiona}
        \pgfmathsetmacro{\b}{\pgfplotstableregressionb}
        
        \addplot[very thick, black, domain=0:1] {\m*x + \b};
        
        \end{axis}
        \end{tikzpicture}
        \label{fig:improvement_nmi:overview}
    \end{subfigure}
    \begin{subfigure}[t]{.48\linewidth}
        \centering
        \scalefont{0.8}
        \begin{tikzpicture}
        \begin{axis}[
          width=\linewidth,
          width=4.5cm, height=3.5cm,
          xmin=0.45, xmax=0.9,
          ymin=0.01, ymax=0.06,
          xlabel={$\overline{\rho}$},
          xticklabel style={font=\scriptsize},
          legend style={at={(0.5,1.15)}, anchor=south, legend columns=2, nodes={scale=0.8, transform shape},legend cell align={left}},
          ymajorgrids=true,
          grid style=dashed,
        ]

\addplot[
          only marks, mark=o, mark size=3.2pt,
          draw=black, fill=ETHGray, line width=0.6pt
        ] coordinates {(0.6006349105923761, 0.04052922210013088)};
        
        \addplot[
          only marks, mark=square*, mark size=3.2pt,
          draw=black, fill=ETHPurple, line width=0.6pt
        ] coordinates {(0.6593337999567225, 0.03134322928026256)};
        
        \addplot[
          only marks, mark=triangle*, mark size=3.6pt,
          draw=black, fill=ETHRed, line width=0.6pt
        ] coordinates {(0.7485089153551084, 0.018896999552622917)};
        
        \addplot[
          only marks, mark=diamond*, mark size=3.6pt,
          draw=black, fill=ETHBronze, line width=0.6pt
        ] coordinates {(0.5113431052530255, 0.05535514537067254)};
        
        \addplot[
          only marks, mark=otimes*, mark size=3.6pt,
          draw=black, fill=ETHGreen, line width=0.6pt
        ] coordinates {(0.7145978117844887, 0.02934642553082467)};
        
        \addplot[
          only marks, mark=star, mark size=2.6pt,
          draw=black, fill=ETHPetrol, line width=0.6pt
        ] coordinates {(0.7827027293666758, 0.021071888203437306)};
        
        \addplot[
          only marks, mark=pentagon*, mark size=3.6pt,
          draw=black, fill=ETHBlue, line width=0.6pt
        ] coordinates {(0.808790397980558, 0.020040048892834172)};
        
        \pgfplotstableread{
        x   y
0.6006349105923761 0.04052922210013088
0.6593337999567225 0.03134322928026256
0.7485089153551084 0.018896999552622917
0.5113431052530255 0.05535514537067254
0.7145978117844887 0.02934642553082467
0.7827027293666758 0.021071888203437306
0.808790397980558 0.020040048892834172
        }\allpoints
        
        \pgfplotstablecreatecol[linear regression={y=y,x=x}]{regression}{\allpoints}
        \pgfmathsetmacro{\m}{\pgfplotstableregressiona}
        \pgfmathsetmacro{\b}{\pgfplotstableregressionb}
        
        \addplot[very thick, black, domain=0:1] {\m*x + \b};
        
        \end{axis}
        \end{tikzpicture}
        \label{fig:improvement_corr:overview}
    \end{subfigure}
    \vspace{-.27cm}
    \caption{Majority voting improvement decreases with higher prediction dependency. Across models, we can find that voting improvement $\Delta A_{\mathrm{MV}}(16)$ is negatively correlated with both $\overline{\mathrm{NMI}}$ and $\overline{\rho}$, confirming theoretical predictions.}
    \label{fig:improvement_both:overview}
\end{figure}
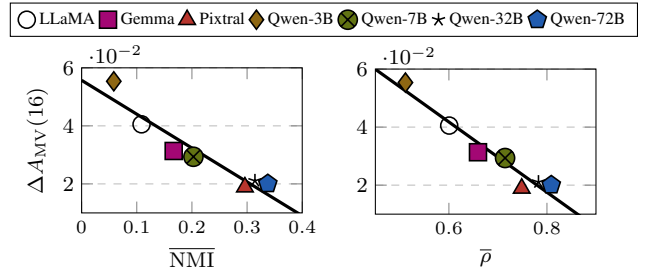

%% file: figures/table_entropyvsmv_samesize.tex
\begin{table*}[!ht]
\centering
\caption{Comparison of ETTC and Voting in the multi-model ensemble setting with \emph{similar-sized models from different families}. ETTC consistently outperforms majority voting across all six datasets, with particularly large gains on benchmarks where model accuracies vary widely (e.g., MathVista, MathVision). This highlights ETTC's ability to prioritize stronger models when aggregating predictions.}
\begin{tabular}{lccccccc}
\toprule
\midrule
\multirow{2}{*}{\textbf{Accuracy (\%)}} & \multicolumn{4}{c}{\textbf{Models}} & \multirow{2}{*}{\textbf{Average}} & \multirow{2}{*}{\textbf{Voting}} & \multirow{2}{*}{\textbf{ETTC}} \\
\cmidrule(lr){2-5}
 & \textbf{LLaMA} & \textbf{Pixtral} & \textbf{Gemma} & \textbf{Qwen-7B} &  &  &  \\
\midrule
MathVista  & 52.04& 56.03& 65.03& \underline{72.08}& 61.30& 68.33& \textbf{75.93}\\
MathVision & 23.41& 25.20& 31.84& 30.18& 27.66& \underline{32.05} & \textbf{35.57}\\
TQA        & 70.41& 77.34& 78.86& 78.50& 76.28& \underline{83.65}& \textbf{83.90}\\
ScienceQA  & 77.84& 78.32& 77.83& 79.76& 78.44& \textbf{85.52}& \underline{85.28}\\
MMStar     & 46.09& 50.35& 53.40& 56.77& 51.65& \underline{59.27}& \textbf{60.07}\\
MMMU       & 42.87& 47.65& 52.49& 50.53& 48.39& \underline{53.66}& \textbf{58.63}\\
\midrule
\textbf{Average} & 52.11& 55.82& 59.91& 61.30& 57.29& \underline{63.75} & \textbf{66.56}\\
\midrule
\bottomrule
\end{tabular}
\label{tab:entropyvsmv:samesize}
\end{table*}

%% file: figures/table_entropyvsmv_samefamily.tex
\begin{table*}[!ht]
\centering
\caption{Comparison of ETTC and Voting in the multi-model ensemble setting using \emph{same-family models} (Qwen) of increasing scale. ETTC consistently outperforms voting across all datasets, even under highly correlated predictions. Gains are especially pronounced when model accuracies increase with scale, demonstrating ETTC's advantage in prioritizing stronger models within homogeneous ensembles.}
\begin{tabular}{lccccccc}
\toprule
\midrule
\multirow{2}{*}{\textbf{Accuracy (\%)}} & \multicolumn{4}{c}{\textbf{Models}} & \multirow{2}{*}{\textbf{Average}} & \multirow{2}{*}{\textbf{Voting}} & \multirow{2}{*}{\textbf{ETTC}} \\
\cmidrule(lr){2-5}
~ & \textbf{Qwen-3B} & \textbf{Qwen-7B} & \textbf{Qwen-32B} & \textbf{Qwen-72B} &  &  &  \\
\midrule
MathVista  & 51.94& 72.08& 78.58& 80.58& 70.80& \underline{83.15} & \textbf{84.44}\\
MathVision & 22.27& 30.18& 38.80& \underline{42.89} & 33.53& 41.32& \textbf{44.84}\\
TQA        & 60.85& 78.50& 83.06& 84.52& 76.73& \underline{84.90} & \textbf{86.70}\\
ScienceQA  & 66.67& 79.76& 84.21& \underline{84.64} & 78.82& 84.04& \textbf{85.03}\\
MMStar     & 41.22& 56.77& 56.34& \underline{62.56} & 54.22& 61.00& \textbf{63.73}\\
MMMU       & 37.41& 50.53& 59.04& \underline{64.18} & 52.79& 58.63& \textbf{65.34}\\
\midrule
\textbf{Average} & 46.73& 61.30& 66.67& \underline{69.90} & 61.15& 68.84& \textbf{71.68}\\
\midrule
\bottomrule
\end{tabular}
\label{tab:entropyvsmv:samefamily}
\end{table*}

%% file: figures/table_dataset_details.tex
\begin{table}[ht]
\small
\centering
\caption{Dataset statistics and characteristics used in our evaluation. Each dataset is categorized by its domain (Math, Diagram, or General), the evaluation split used (e.g., test or validation), the number of multiple-choice questions (\textbf{Size}), and the number of answer options per question (\textbf{Option Num.}).}
\begin{tabular}{lcccc}
\toprule
\midrule
\textbf{Dataset} & \textbf{Domain} & \textbf{Type} & \textbf{Size} & \textbf{Option Num.}\\
\midrule
MathVista & Math & testmini & 540 & 2--8 \\
MathVision & Math & test & 1,532 & 5 \\
\midrule
TQA & Diagram & test & 3,285 & 4 \\
ScienceQA & Diagram & test & 2,017 & 2--5 \\
\midrule
MMStar & General & val & 1,500 & 4 \\
MMMU & General & val & 805 & 2--9 \\
\midrule
\bottomrule
\end{tabular}
\label{tab:dataset_stats}
\end{table}

%% file: figures/appendix/prompt_ncot.tex
\begin{figure}[ht]
\begin{tcolorbox}[colback=violet!10, colframe=violet!70!black, title={\large Pipeline for Benchmark Evaluation}]

\begin{center}
    \includegraphics[width=0.98\textwidth]{images/appendix/case_analysis_math_14.png}
\end{center}

\begin{tcolorbox}[colback=gray!10, colframe=gray!80]
\setstretch{1.1}

\textbf{Prompt: }

Question: Is Periwinkle the maximum? Options: 'yes', 'no'. ONLY SIMPLY choose the correct option in a single sentence or a single word. No preamble, no explanation.
\end{tcolorbox}

\begin{tcolorbox}[colback=gray!10, colframe=gray!80]
\setstretch{1.1}
\textbf{Output: }

No.
\end{tcolorbox}

\end{tcolorbox}
\caption{Example of a direct QA prompt used for evaluating model predictions without reasoning.}
\label{fig:prompt_noct}
\end{figure}

%% file: figures/appendix/prompt_cot.tex
\begin{figure}[!t]
\centering
\begin{tcolorbox}[colback=violet!10, colframe=violet!70!black, title={\large Pipeline for Benchmark Evaluation}]

\begin{center}
    \includegraphics[width=0.98\textwidth]{images/appendix/case_analysis_math_14.png}
\end{center}

\begin{tcolorbox}[colback=gray!10, colframe=gray!80]
\setstretch{1.1}

\textbf{Prompt: }

Question: Is Periwinkle the maximum? Options: 'yes', 'no'. Among the given options, the answer is: (Let's think step by step, and give the answer at the end of your thought with **Answer:**).
\end{tcolorbox}

\begin{tcolorbox}[colback=gray!10, colframe=gray!80]
\setstretch{1.1}
\textbf{Output: }

Looking at the bar plot, I can see that the y-values for the colors are: Dark Salmon (48), Periwinkle (60), Gray (92), and Tomato (96). \textbackslash{}n\textbackslash{}nComparing these values, I can see that the highest value is 96, which corresponds to the color Tomato. \textbackslash{}n\textbackslash{}nTherefore, the maximum is Tomato.\textbackslash{}n\textbackslash{}n**Answer:** no
\end{tcolorbox}

\end{tcolorbox}
\caption{Example of a chain-of-thought (CoT) prompt used to elicit intermediate reasoning steps. This format is used when analyzing consistency or measuring correctness under step-by-step reasoning.}
\label{fig:prompt_cot}
\end{figure}

%% file: figures/appendix/pivot_phrases.tex
\begin{table}[ht]
\centering
\small
\caption{Pivot phrases categorized by reasoning function.}
\begin{tabular}{p{3.2cm} p{9.8cm}}
\toprule
\midrule
\textbf{Reasoning Type} & \textbf{Example Phrases} \\
\midrule
\textbf{Realization} 
& ``wait'', ``oh'', ``actually'', ``I missed something''\\
\addlinespace

\textbf{Verification} 
& ``let me doublecheck'',
``to verify'',
``checking again''\\
\addlinespace

\textbf{Exploration} 
& ``what if'',
``another way to look at this'',
``alternatively''\\
\addlinespace

\textbf{Integration} 
& ``now I see how'',
``this connects back to'',
``putting this together''\\

\midrule
\bottomrule

\end{tabular}
\label{tab:pivot}
\end{table}

%% file: figures/appendix/vague_phrases.tex
\begin{table}[!ht]
\small
\centering
\caption{Vague expressions used in model reasoning, grouped by rhetorical effect.}
\begin{tabular}{p{3.2cm} p{9.8cm}}
\toprule
\midrule
\textbf{Reasoning Type} & \textbf{Example Phrases} \\
\midrule
\textbf{Uncertainty}& ``maybe"'', ``possibly'', ``perhaps'', ``probably'', ``might be'', ``could be'', 
            ``it seems''\\
\addlinespace

\textbf{Hedging}& ``somewhat'', ``rather'', ``kind of'', ``sort of'', 
            ``generally'', ``typically''\\
\midrule
\bottomrule
\end{tabular}
\label{tab:vague}
\end{table}

%% file: figures/appendix/feature_all.tex
\begin{table}[!ht]
\small
\centering
\caption{Overview of lexical and stylistic features used for CoT-based prediction.}
\begin{tabular}{p{3.2cm} p{9.8cm}}
\toprule
\midrule
\textbf{Feature} & \textbf{Modeling Method} \\
\midrule
\textbf{Token Number} & Measures the number of tokens in the CoT response. Longer responses may indicate more reasoning steps, though excessive length may signal loops or noise. We vectorize it as ${1}/\text{Token Number}$. \\
\addlinespace
\textbf{Lexical Diversity} & Captures vocabulary richness by counting the number of unique tokens. Low diversity often suggests repetition. We vectorize it as ${1}/\text{Vocabulary Size}$. \\
\addlinespace
\textbf{Pivot Word Number} & Counts the number of pivot expressions from \Cref{tab:pivot}, indicating structured reasoning or correction. We vectorize it as ${1}/\text{Pivot Word Number}$. \\
\addlinespace
\textbf{Vague Word Number} & Counts the number of vague phrases from \Cref{tab:vague}, which may reflect uncertainty or low confidence. We vectorize it as $1 - {1}/\text{Vague Word Number}$.  \\
\midrule
\bottomrule
\end{tabular}
\label{tab:featureall}
\end{table}

%% file: figures/improvement_nmi_all.tex
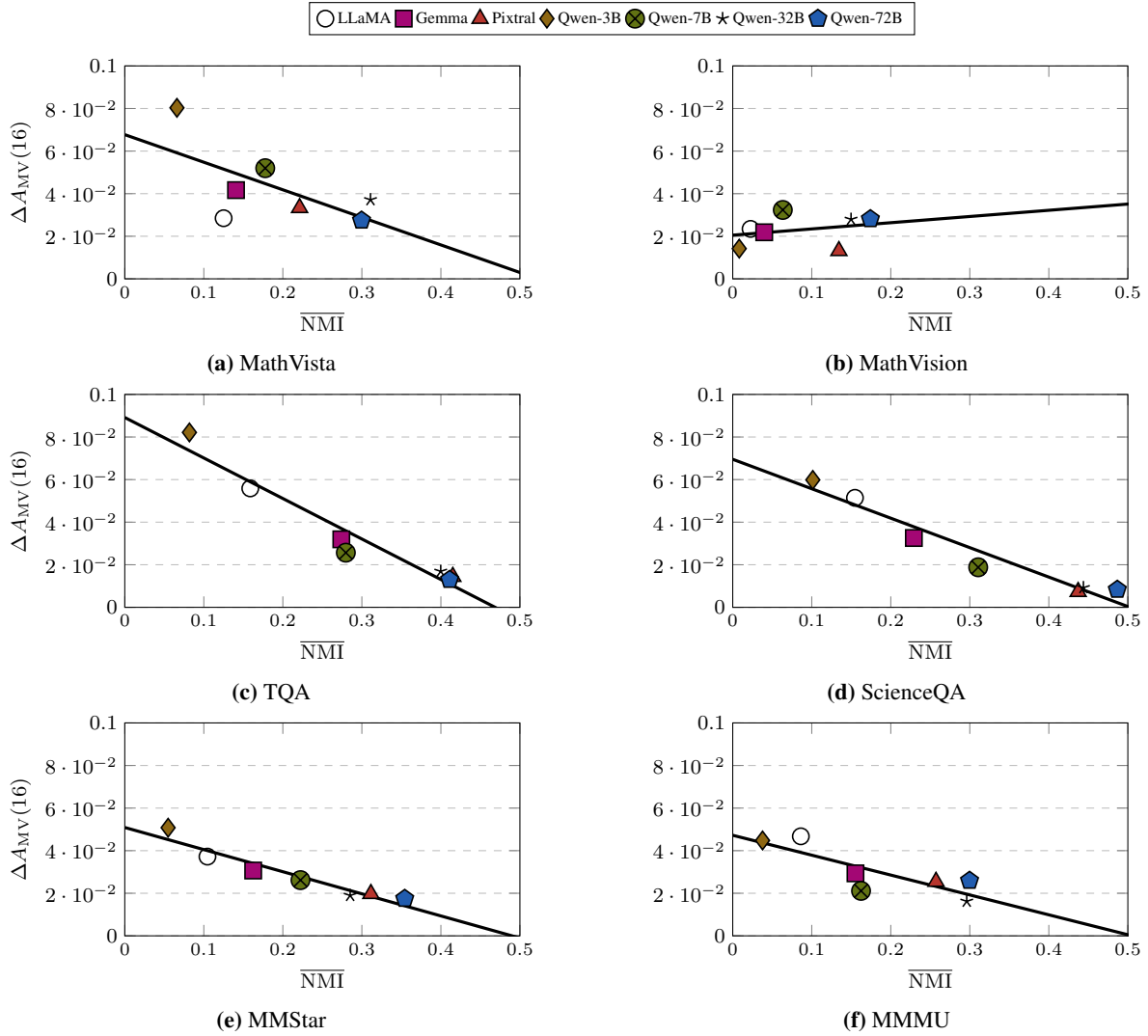
\begin{figure}[!ht]
\vspace{.8cm}
    \centering
    \begin{subfigure}[t]{.495\textwidth}
        \centering
        \scalefont{0.8}
        \begin{tikzpicture}
        \begin{axis}[
          width=\linewidth,
          width=7cm, height=4.5cm,
          xmin=0, xmax=0.5,
          ymin=0, ymax=0.1,
          xlabel={$\overline{\mathrm{NMI}}$},
          ylabel={$\Delta A_{\mathrm{MV}}(16)$},
          xticklabel style={font=\scriptsize},
          legend style={at={(2.,1.3)}, legend columns=8, nodes={scale=0.8, transform shape},/tikz/overlay,},
          ymajorgrids=true,
          grid style=dashed,
        ]
        
        \addplot[
          only marks, mark=o, mark size=3.2pt,
          draw=black, fill=ETHGray, line width=0.6pt
        ] coordinates {(0.12506253773210915, 0.028472222222222232)};
        \addlegendentry{LLaMA}
        
        \addplot[
          only marks, mark=square*, mark size=3.2pt,
          draw=black, fill=ETHPurple, line width=0.6pt
        ] coordinates {(0.1407469971730799, 0.04166666666666652)};
        \addlegendentry{Gemma}
        
        \addplot[
          only marks, mark=triangle*, mark size=3.6pt,
          draw=black, fill=ETHRed, line width=0.6pt
        ] coordinates {(0.22121046296895408, 0.03333333333333344)};
        \addlegendentry{Pixtral}
        
        \addplot[
          only marks, mark=diamond*, mark size=3.6pt,
          draw=black, fill=ETHBronze, line width=0.6pt
        ] coordinates {(0.06600692889671005, 0.08032407407407405)};
        \addlegendentry{Qwen-3B}
        
        \addplot[
          only marks, mark=otimes*, mark size=3.6pt,
          draw=black, fill=ETHGreen, line width=0.6pt
        ] coordinates {(0.17772193818350726, 0.05196759259259254)};
        \addlegendentry{Qwen-7B}
        
        \addplot[
          only marks, mark=star, mark size=2.6pt,
          draw=black, fill=ETHPetrol, line width=0.6pt
        ] coordinates {(0.3109096605248164, 0.03715277777777781)};
        \addlegendentry{Qwen-32B}
        
        \addplot[
          only marks, mark=pentagon*, mark size=3.6pt,
          draw=black, fill=ETHBlue, line width=0.6pt
        ] coordinates {(0.29960633223076366, 0.027430555555555514)};
        \addlegendentry{Qwen-72B}
        
        \pgfplotstableread{
        x   y
0.12506253773210915 0.028472222222222232
0.1407469971730799 0.04166666666666652
0.22121046296895408 0.03333333333333344
0.06600692889671005 0.08032407407407405
0.17772193818350726 0.05196759259259254
0.3109096605248164 0.03715277777777781
0.29960633223076366 0.027430555555555514
        }\allpoints
        
        \pgfplotstablecreatecol[linear regression={y=y,x=x}]{regression}{\allpoints}
        \pgfmathsetmacro{\m}{\pgfplotstableregressiona}
        \pgfmathsetmacro{\b}{\pgfplotstableregressionb}
        
        \addplot[very thick, black, domain=0:1] {\m*x + \b};
        
        \end{axis}
        \end{tikzpicture}
        \caption{MathVista}
        \label{fig:improvement_nmi:mathvista}
    \end{subfigure}
    \begin{subfigure}[t]{.495\textwidth}
        \centering
        \scalefont{0.8}
        \begin{tikzpicture}
        \begin{axis}[
          width=\linewidth,
          width=7cm, height=4.5cm,
          xmin=0, xmax=0.5,
          ymin=0, ymax=0.1,
          xlabel={$\overline{\mathrm{NMI}}$},
          xticklabel style={font=\scriptsize},
          legend style={at={(0.5,1.15)}, anchor=south, legend columns=2, nodes={scale=0.8, transform shape},legend cell align={left}},
          ymajorgrids=true,
          grid style=dashed,
        ]
        
        \addplot[
          only marks, mark=o, mark size=3.2pt,
          draw=black, fill=ETHGray, line width=0.6pt
        ] coordinates {(0.02259082046024646, 0.023457898172323854)};
        
        \addplot[
          only marks, mark=square*, mark size=3.2pt,
          draw=black, fill=ETHPurple, line width=0.6pt
        ] coordinates {(0.039855736721699925, 0.021866840731070536)};
        
        \addplot[
          only marks, mark=triangle*, mark size=3.6pt,
          draw=black, fill=ETHRed, line width=0.6pt
        ] coordinates {(0.1343662345124931, 0.013136422976501305)};
        
        \addplot[
          only marks, mark=diamond*, mark size=3.6pt,
          draw=black, fill=ETHBronze, line width=0.6pt
        ] coordinates {(0.008241235443607615, 0.014197127937336879)};
        
        \addplot[
          only marks, mark=otimes*, mark size=3.6pt,
          draw=black, fill=ETHGreen, line width=0.6pt
        ] coordinates {(0.06336350628932781, 0.03235150130548303)};
        
        \addplot[
          only marks, mark=star, mark size=2.6pt,
          draw=black, fill=ETHPetrol, line width=0.6pt
        ] coordinates {(0.1497352914466856, 0.027945496083550958)};
        
        \addplot[
          only marks, mark=pentagon*, mark size=3.6pt,
          draw=black, fill=ETHBlue, line width=0.6pt
        ] coordinates {(0.17419624459518446, 0.028108681462141072)};
        
        \pgfplotstableread{
        x   y
0.02259082046024646 0.023457898172323854
0.039855736721699925 0.021866840731070536
0.1343662345124931 0.013136422976501305
0.008241235443607615 0.014197127937336879
0.06336350628932781 0.03235150130548303
0.1497352914466856 0.027945496083550958
0.17419624459518446 0.028108681462141072
        }\allpoints
        
        \pgfplotstablecreatecol[linear regression={y=y,x=x}]{regression}{\allpoints}
        \pgfmathsetmacro{\m}{\pgfplotstableregressiona}
        \pgfmathsetmacro{\b}{\pgfplotstableregressionb}
        
        \addplot[very thick, black, domain=0:1] {\m*x + \b};
        
        \end{axis}
        \end{tikzpicture}
        \caption{MathVision}
        \label{fig:improvement_nmi:mathvision}
    \end{subfigure}
    \begin{subfigure}[t]{.495\textwidth}
        \centering
        \scalefont{0.8}
        \begin{tikzpicture}
        \begin{axis}[
          width=\linewidth,
          width=7cm, height=4.5cm,
          xmin=0, xmax=0.5,
          ymin=0, ymax=0.1,
          xlabel={$\overline{\mathrm{NMI}}$},
          ylabel={$\Delta A_{\mathrm{MV}}(16)$},
          xticklabel style={font=\scriptsize},
          legend style={at={(0.5,1.15)}, anchor=south, legend columns=2, nodes={scale=0.8, transform shape},legend cell align={left}},
          ymajorgrids=true,
          grid style=dashed,
        ]
        
        \addplot[
          only marks, mark=o, mark size=3.2pt,
          draw=black, fill=ETHGray, line width=0.6pt
        ] coordinates {(0.15879199656075815, 0.05585996955859973)};
        
        \addplot[
          only marks, mark=square*, mark size=3.2pt,
          draw=black, fill=ETHPurple, line width=0.6pt
        ] coordinates {(0.27400923033362734, 0.03194444444444455)};
        
        \addplot[
          only marks, mark=triangle*, mark size=3.6pt,
          draw=black, fill=ETHRed, line width=0.6pt
        ] coordinates {(0.4151015698496539, 0.014383561643835363)};
        
        \addplot[
          only marks, mark=diamond*, mark size=3.6pt,
          draw=black, fill=ETHBronze, line width=0.6pt
        ] coordinates {(0.08178600738319038, 0.08215372907153728)};
        
        \addplot[
          only marks, mark=otimes*, mark size=3.6pt,
          draw=black, fill=ETHGreen, line width=0.6pt
        ] coordinates {(0.279719349269899, 0.02566590563165916)};
        
        \addplot[
          only marks, mark=star, mark size=2.6pt,
          draw=black, fill=ETHPetrol, line width=0.6pt
        ] coordinates {(0.39980238891141406, 0.016914003044140036)};
        
        \addplot[
          only marks, mark=pentagon*, mark size=3.6pt,
          draw=black, fill=ETHBlue, line width=0.6pt
        ] coordinates {(0.41105837909282467, 0.012975646879756386)};
        
        \pgfplotstableread{
        x   y
0.15879199656075815 0.05585996955859973
0.27400923033362734 0.03194444444444455
0.4151015698496539 0.014383561643835363
0.08178600738319038 0.08215372907153728
0.279719349269899 0.02566590563165916
0.39980238891141406 0.016914003044140036
0.41105837909282467 0.012975646879756386
        }\allpoints
        
        \pgfplotstablecreatecol[linear regression={y=y,x=x}]{regression}{\allpoints}
        \pgfmathsetmacro{\m}{\pgfplotstableregressiona}
        \pgfmathsetmacro{\b}{\pgfplotstableregressionb}
        
        \addplot[very thick, black, domain=0:1] {\m*x + \b};
        
        \end{axis}
        \end{tikzpicture}
        \caption{TQA}
        \label{fig:improvement_nmi:tqa}
    \end{subfigure}
    \begin{subfigure}[t]{.495\textwidth}
        \centering
        \scalefont{0.8}
        \begin{tikzpicture}
        \begin{axis}[
          width=\linewidth,
          width=7cm, height=4.5cm,
          xmin=0, xmax=0.5,
          ymin=0, ymax=0.1,
          xlabel={$\overline{\mathrm{NMI}}$},
          xticklabel style={font=\scriptsize},
          legend style={at={(0.5,1.15)}, anchor=south, legend columns=2, nodes={scale=0.8, transform shape},legend cell align={left}},
          ymajorgrids=true,
          grid style=dashed,
        ]
        
        \addplot[
          only marks, mark=o, mark size=3.2pt,
          draw=black, fill=ETHGray, line width=0.6pt
        ] coordinates {(0.15466061033373482, 0.051437778879524054)};
        
        \addplot[
          only marks, mark=square*, mark size=3.2pt,
          draw=black, fill=ETHPurple, line width=0.6pt
        ] coordinates {(0.22921294389732752, 0.03256693108577091)};
        
        \addplot[
          only marks, mark=triangle*, mark size=3.6pt,
          draw=black, fill=ETHRed, line width=0.6pt
        ] coordinates {(0.43710518732658876, 0.007312840852751523)};
        
        \addplot[
          only marks, mark=diamond*, mark size=3.6pt,
          draw=black, fill=ETHBronze, line width=0.6pt
        ] coordinates {(0.10136534695636255, 0.059866137828457955)};
        
        \addplot[
          only marks, mark=otimes*, mark size=3.6pt,
          draw=black, fill=ETHGreen, line width=0.6pt
        ] coordinates {(0.3106278001026787, 0.018808874566187295)};
        
        \addplot[
          only marks, mark=star, mark size=2.6pt,
          draw=black, fill=ETHPetrol, line width=0.6pt
        ] coordinates {(0.4437782926073149, 0.009234010907288037)};
        
        \addplot[
          only marks, mark=pentagon*, mark size=3.6pt,
          draw=black, fill=ETHBlue, line width=0.6pt
        ] coordinates {(0.4866648638901999, 0.008335399107585517)};
        
        \pgfplotstableread{
        x   y
0.15466061033373482 0.051437778879524054
0.22921294389732752 0.03256693108577091
0.43710518732658876 0.007312840852751523
0.10136534695636255 0.059866137828457955
0.3106278001026787 0.018808874566187295
0.4437782926073149 0.009234010907288037
0.4866648638901999 0.008335399107585517
        }\allpoints
        
        \pgfplotstablecreatecol[linear regression={y=y,x=x}]{regression}{\allpoints}
        \pgfmathsetmacro{\m}{\pgfplotstableregressiona}
        \pgfmathsetmacro{\b}{\pgfplotstableregressionb}
        
        \addplot[very thick, black, domain=0:1] {\m*x + \b};
        
        \end{axis}
        \end{tikzpicture}
        \caption{ScienceQA}
        \label{fig:improvement_nmi:scienceqa}
    \end{subfigure}
    \begin{subfigure}[t]{.495\textwidth}
        \centering
        \scalefont{0.8}
        \begin{tikzpicture}
        \begin{axis}[
          width=\linewidth,
          width=7cm, height=4.5cm,
          xmin=0, xmax=0.5,
          ymin=0, ymax=0.1,
          xlabel={$\overline{\mathrm{NMI}}$},
          ylabel={$\Delta A_{\mathrm{MV}}(16)$},
          xticklabel style={font=\scriptsize},
          legend style={at={(0.5,1.15)}, anchor=south, legend columns=2, nodes={scale=0.8, transform shape},legend cell align={left}},
          ymajorgrids=true,
          grid style=dashed,
        ]
        
        \addplot[
          only marks, mark=o, mark size=3.2pt,
          draw=black, fill=ETHGray, line width=0.6pt
        ] coordinates {(0.10461348726527127, 0.03720833333333334)};
        
        \addplot[
          only marks, mark=square*, mark size=3.2pt,
          draw=black, fill=ETHPurple, line width=0.6pt
        ] coordinates {(0.1624938716462898, 0.03066666666666662)};
        
        \addplot[
          only marks, mark=triangle*, mark size=3.6pt,
          draw=black, fill=ETHRed, line width=0.6pt
        ] coordinates {(0.3114410705765362, 0.019750000000000156)};
        
        \addplot[
          only marks, mark=diamond*, mark size=3.6pt,
          draw=black, fill=ETHBronze, line width=0.6pt
        ] coordinates {(0.05483662197073346, 0.05079166666666668)};
        
        \addplot[
          only marks, mark=otimes*, mark size=3.6pt,
          draw=black, fill=ETHGreen, line width=0.6pt
        ] coordinates {(0.22235049098516607, 0.02616666666666656)};
        
        \addplot[
          only marks, mark=star, mark size=2.6pt,
          draw=black, fill=ETHPetrol, line width=0.6pt
        ] coordinates {(0.2851949458699199, 0.0189583333333333)};
        
        \addplot[
          only marks, mark=pentagon*, mark size=3.6pt,
          draw=black, fill=ETHBlue, line width=0.6pt
        ] coordinates {(0.35425011574486026, 0.017458333333333242)};
        
        \pgfplotstableread{
        x   y
0.10461348726527127 0.03720833333333334
0.1624938716462898 0.03066666666666662
0.3114410705765362 0.019750000000000156
0.05483662197073346 0.05079166666666668
0.22235049098516607 0.02616666666666656
0.2851949458699199 0.0189583333333333
0.35425011574486026 0.017458333333333242
        }\allpoints
        
        \pgfplotstablecreatecol[linear regression={y=y,x=x}]{regression}{\allpoints}
        \pgfmathsetmacro{\m}{\pgfplotstableregressiona}
        \pgfmathsetmacro{\b}{\pgfplotstableregressionb}
        
        \addplot[very thick, black, domain=0:1] {\m*x + \b};
        
        \end{axis}
        \end{tikzpicture}
        \caption{MMStar}
        \label{fig:improvement_nmi:mmstar}
    \end{subfigure}
    \begin{subfigure}[t]{.495\textwidth}
        \centering
        \scalefont{0.8}
        \begin{tikzpicture}
        \begin{axis}[
          width=\linewidth,
          width=7cm, height=4.5cm,
          xmin=0, xmax=0.5,
          ymin=0, ymax=0.1,
          xlabel={$\overline{\mathrm{NMI}}$},
          xticklabel style={font=\scriptsize},
          legend style={at={(0.5,1.15)}, anchor=south, legend columns=2, nodes={scale=0.8, transform shape},legend cell align={left}},
          ymajorgrids=true,
          grid style=dashed,
        ]
        
        \addplot[
          only marks, mark=o, mark size=3.2pt,
          draw=black, fill=ETHGray, line width=0.6pt
        ] coordinates {(0.0861867362694768, 0.04673913043478267)};
        
        \addplot[
          only marks, mark=square*, mark size=3.2pt,
          draw=black, fill=ETHPurple, line width=0.6pt
        ] coordinates {(0.1553276086750733, 0.02934782608695652)};
        
        \addplot[
          only marks, mark=triangle*, mark size=3.6pt,
          draw=black, fill=ETHRed, line width=0.6pt
        ] coordinates {(0.2572348187699728, 0.02546583850931672)};
        
        \addplot[
          only marks, mark=diamond*, mark size=3.6pt,
          draw=black, fill=ETHBronze, line width=0.6pt
        ] coordinates {(0.037761875912948266, 0.04479813664596277)};
        
        \addplot[
          only marks, mark=otimes*, mark size=3.6pt,
          draw=black, fill=ETHGreen, line width=0.6pt
        ] coordinates {(0.16238464104305833, 0.02111801242236022)};
        
        \addplot[
          only marks, mark=star, mark size=2.6pt,
          draw=black, fill=ETHPetrol, line width=0.6pt
        ] coordinates {(0.2961227520348752, 0.01622670807453408)};
        
        \addplot[
          only marks, mark=pentagon*, mark size=3.6pt,
          draw=black, fill=ETHBlue, line width=0.6pt
        ] coordinates {(0.2996938360655412, 0.025931677018633637)};
        
        \pgfplotstableread{
        x   y
0.0861867362694768 0.04673913043478267
0.1553276086750733 0.02934782608695652
0.2572348187699728 0.02546583850931672
0.037761875912948266 0.04479813664596277
0.16238464104305833 0.02111801242236022
0.2961227520348752 0.01622670807453408
0.2996938360655412 0.025931677018633637
        }\allpoints
        
        \pgfplotstablecreatecol[linear regression={y=y,x=x}]{regression}{\allpoints}
        \pgfmathsetmacro{\m}{\pgfplotstableregressiona}
        \pgfmathsetmacro{\b}{\pgfplotstableregressionb}
        
        \addplot[very thick, black, domain=0:1] {\m*x + \b};
        
        \end{axis}
        \end{tikzpicture}
        \caption{MMMU}
        \label{fig:improvement_nmi:mmmu}
    \end{subfigure}
    \caption{Majority voting improvement $\Delta A_{\mathrm{MV}}(16)$ plotted against average pairwise normalized mutual information ($\overline{\mathrm{NMI}}$) for each model on each dataset. A negative trend suggests that higher prediction dependency reduces the benefit of majority voting.}
    \label{fig:improvement_nmi:all}
\end{figure}

%% file: figures/improvement_corr_all.tex
\begin{figure}[!ht]
\vspace{.8cm}
    \centering
    \begin{subfigure}[t]{.495\textwidth}
        \centering
        \scalefont{0.8}
        \begin{tikzpicture}
        \begin{axis}[
          width=\linewidth,
          width=7cm, height=4.5cm,
          xmin=0.4, xmax=1.0,
          ymin=0, ymax=0.1,
          xlabel={$\overline{\rho}$},
          ylabel={$\Delta A_{\mathrm{MV}}(16)$},
          xticklabel style={font=\scriptsize},
          legend style={at={(2.,1.3)}, legend columns=8, nodes={scale=0.8, transform shape},/tikz/overlay,},
          ymajorgrids=true,
          grid style=dashed,
        ]
        
        \addplot[
          only marks, mark=o, mark size=3.2pt,
          draw=black, fill=ETHGray, line width=0.6pt
        ] coordinates {(0.5449761036644336, 0.028472222222222232)};
        \addlegendentry{LLaMA}
        
        \addplot[
          only marks, mark=square*, mark size=3.2pt,
          draw=black, fill=ETHPurple, line width=0.6pt
        ] coordinates {(0.5546185741946377, 0.04166666666666652)};
        \addlegendentry{Gemma}
        
        \addplot[
          only marks, mark=triangle*, mark size=3.6pt,
          draw=black, fill=ETHRed, line width=0.6pt
        ] coordinates {(0.6811205660842521, 0.03333333333333344)};
        \addlegendentry{Pixtral}
        
        \addplot[
          only marks, mark=diamond*, mark size=3.6pt,
          draw=black, fill=ETHBronze, line width=0.6pt
        ] coordinates {(0.4787017305897051, 0.08032407407407405)};
        \addlegendentry{Qwen-3B}
        
        \addplot[
          only marks, mark=otimes*, mark size=3.6pt,
          draw=black, fill=ETHGreen, line width=0.6pt
        ] coordinates {(0.7198648121890151, 0.05196759259259254)};
        \addlegendentry{Qwen-7B}
        
        \addplot[
          only marks, mark=star, mark size=2.6pt,
          draw=black, fill=ETHPetrol, line width=0.6pt
        ] coordinates {(0.8067338894054553, 0.03715277777777781)};
        \addlegendentry{Qwen-32B}
        
        \addplot[
          only marks, mark=pentagon*, mark size=3.6pt,
          draw=black, fill=ETHBlue, line width=0.6pt
        ] coordinates {(0.7992185067827278, 0.027430555555555514)};
        \addlegendentry{Qwen-72B}
        
        \pgfplotstableread{
        x   y
0.5449761036644336 0.028472222222222232
0.5546185741946377 0.04166666666666652
0.6811205660842521 0.03333333333333344
0.4787017305897051 0.08032407407407405
0.7198648121890151 0.05196759259259254
0.8067338894054553 0.03715277777777781
0.7992185067827278 0.027430555555555514
        }\allpoints
        
        \pgfplotstablecreatecol[linear regression={y=y,x=x}]{regression}{\allpoints}
        \pgfmathsetmacro{\m}{\pgfplotstableregressiona}
        \pgfmathsetmacro{\b}{\pgfplotstableregressionb}
        
        \addplot[very thick, black, domain=0:1] {\m*x + \b};
        
        \end{axis}
        \end{tikzpicture}
        \caption{MathVista}
        \label{fig:improvement_corr:mathvista}
    \end{subfigure}
    \begin{subfigure}[t]{.495\textwidth}
        \centering
        \scalefont{0.8}
        \begin{tikzpicture}
        \begin{axis}[
          width=\linewidth,
          width=7cm, height=4.5cm,
          xmin=0.4, xmax=1.0,
          ymin=0, ymax=0.1,
          xlabel={$\overline{\rho}$},
          xticklabel style={font=\scriptsize},
          legend style={at={(0.5,1.15)}, anchor=south, legend columns=2, nodes={scale=0.8, transform shape},legend cell align={left}},
          ymajorgrids=true,
          grid style=dashed,
        ]
        
        \addplot[
          only marks, mark=o, mark size=3.2pt,
          draw=black, fill=ETHGray, line width=0.6pt
        ] coordinates {(0.6069284023788312, 0.023457898172323854)};
        
        \addplot[
          only marks, mark=square*, mark size=3.2pt,
          draw=black, fill=ETHPurple, line width=0.6pt
        ] coordinates {(0.6559951638675384, 0.021866840731070536)};
        
        \addplot[
          only marks, mark=triangle*, mark size=3.6pt,
          draw=black, fill=ETHRed, line width=0.6pt
        ] coordinates {(0.6601019758562058, 0.013136422976501305)};
        
        \addplot[
          only marks, mark=diamond*, mark size=3.6pt,
          draw=black, fill=ETHBronze, line width=0.6pt
        ] coordinates {(0.566734702522583, 0.014197127937336879)};
        
        \addplot[
          only marks, mark=otimes*, mark size=3.6pt,
          draw=black, fill=ETHGreen, line width=0.6pt
        ] coordinates {(0.631474691236384, 0.03235150130548303)};
        
        \addplot[
          only marks, mark=star, mark size=2.6pt,
          draw=black, fill=ETHPetrol, line width=0.6pt
        ] coordinates {(0.6551299674220367, 0.027945496083550958)};
        
        \addplot[
          only marks, mark=pentagon*, mark size=3.6pt,
          draw=black, fill=ETHBlue, line width=0.6pt
        ] coordinates {(0.7153838682998538, 0.028108681462141072)};

        \pgfplotstableread{
        x   y
0.6069284023788312 0.023457898172323854
0.6559951638675384 0.021866840731070536
0.6601019758562058 0.013136422976501305
0.566734702522583 0.014197127937336879
0.631474691236384 0.03235150130548303
0.6551299674220367 0.027945496083550958
0.7153838682998538 0.028108681462141072
        }\allpoints
        
        \pgfplotstablecreatecol[linear regression={y=y,x=x}]{regression}{\allpoints}
        \pgfmathsetmacro{\m}{\pgfplotstableregressiona}
        \pgfmathsetmacro{\b}{\pgfplotstableregressionb}
        
        \addplot[very thick, black, domain=0:1] {\m*x + \b};
        
        \end{axis}
        \end{tikzpicture}
        \caption{MathVision}
        \label{fig:improvement_corr:mathvision}
    \end{subfigure}
    \begin{subfigure}[t]{.495\textwidth}
        \centering
        \scalefont{0.8}
        \begin{tikzpicture}
        \begin{axis}[
          width=\linewidth,
          width=7cm, height=4.5cm,
          xmin=0.4, xmax=1.0,
          ymin=0, ymax=0.1,
          xlabel={$\overline{\rho}$},
          ylabel={$\Delta A_{\mathrm{MV}}(16)$},
          xticklabel style={font=\scriptsize},
          legend style={at={(0.5,1.15)}, anchor=south, legend columns=2, nodes={scale=0.8, transform shape},legend cell align={left}},
          ymajorgrids=true,
          grid style=dashed,
        ]
        
        \addplot[
          only marks, mark=o, mark size=3.2pt,
          draw=black, fill=ETHGray, line width=0.6pt
        ] coordinates {(0.6949457797515505, 0.05585996955859973)};
        
        \addplot[
          only marks, mark=square*, mark size=3.2pt,
          draw=black, fill=ETHPurple, line width=0.6pt
        ] coordinates {(0.8009648475881782, 0.03194444444444455)};
        
        \addplot[
          only marks, mark=triangle*, mark size=3.6pt,
          draw=black, fill=ETHRed, line width=0.6pt
        ] coordinates {(0.8631955443618481, 0.014383561643835363)};
        
        \addplot[
          only marks, mark=diamond*, mark size=3.6pt,
          draw=black, fill=ETHBronze, line width=0.6pt
        ] coordinates {(0.5705405383768252, 0.08215372907153728)};
        
        \addplot[
          only marks, mark=otimes*, mark size=3.6pt,
          draw=black, fill=ETHGreen, line width=0.6pt
        ] coordinates {(0.8082136969870573, 0.02566590563165916)};
        
        \addplot[
          only marks, mark=star, mark size=2.6pt,
          draw=black, fill=ETHPetrol, line width=0.6pt
        ] coordinates {(0.8786113800867235, 0.016914003044140036)};
        
        \addplot[
          only marks, mark=pentagon*, mark size=3.6pt,
          draw=black, fill=ETHBlue, line width=0.6pt
        ] coordinates {(0.8861695002475114, 0.012975646879756386)};
        
        \pgfplotstableread{
        x   y
0.6949457797515505 0.05585996955859973
0.8009648475881782 0.03194444444444455
0.8631955443618481 0.014383561643835363
0.5705405383768252 0.08215372907153728
0.8082136969870573 0.02566590563165916
0.8786113800867235 0.016914003044140036
0.8861695002475114 0.012975646879756386
        }\allpoints
        
        \pgfplotstablecreatecol[linear regression={y=y,x=x}]{regression}{\allpoints}
        \pgfmathsetmacro{\m}{\pgfplotstableregressiona}
        \pgfmathsetmacro{\b}{\pgfplotstableregressionb}
        
        \addplot[very thick, black, domain=0:1] {\m*x + \b};
        
        \end{axis}
        \end{tikzpicture}
        \caption{TQA}
        \label{fig:improvement_corr:tqa}
    \end{subfigure}
    \begin{subfigure}[t]{.495\textwidth}
        \centering
        \scalefont{0.8}
        \begin{tikzpicture}
        \begin{axis}[
          width=\linewidth,
          width=7cm, height=4.5cm,
          xmin=0.4, xmax=1.0,
          ymin=0, ymax=0.1,
          xlabel={$\overline{\rho}$},
          xticklabel style={font=\scriptsize},
          legend style={at={(0.5,1.15)}, anchor=south, legend columns=2, nodes={scale=0.8, transform shape},legend cell align={left}},
          ymajorgrids=true,
          grid style=dashed,
        ]
        
        \addplot[
          only marks, mark=o, mark size=3.2pt,
          draw=black, fill=ETHGray, line width=0.6pt
        ] coordinates {(0.6896679246607068, 0.051437778879524054)};
        
        \addplot[
          only marks, mark=square*, mark size=3.2pt,
          draw=black, fill=ETHPurple, line width=0.6pt
        ] coordinates {(0.7481229693538438, 0.03256693108577091)};
        
        \addplot[
          only marks, mark=triangle*, mark size=3.6pt,
          draw=black, fill=ETHRed, line width=0.6pt
        ] coordinates {(0.8459557484151924, 0.007312840852751523)};
        
        \addplot[
          only marks, mark=diamond*, mark size=3.6pt,
          draw=black, fill=ETHBronze, line width=0.6pt
        ] coordinates {(0.551538289243197, 0.059866137828457955)};
        
        \addplot[
          only marks, mark=otimes*, mark size=3.6pt,
          draw=black, fill=ETHGreen, line width=0.6pt
        ] coordinates {(0.8000181320586459, 0.018808874566187295)};
        
        \addplot[
          only marks, mark=star, mark size=2.6pt,
          draw=black, fill=ETHPetrol, line width=0.6pt
        ] coordinates {(0.8809922997535357, 0.009234010907288037)};
        
        \addplot[
          only marks, mark=pentagon*, mark size=3.6pt,
          draw=black, fill=ETHBlue, line width=0.6pt
        ] coordinates {(0.8888203412657188, 0.008335399107585517)};
        
        \pgfplotstableread{
        x   y
0.6896679246607068 0.051437778879524054
0.7481229693538438 0.03256693108577091
0.8459557484151924 0.007312840852751523
0.551538289243197 0.059866137828457955
0.8000181320586459 0.018808874566187295
0.8809922997535357 0.009234010907288037
0.8888203412657188 0.008335399107585517
        }\allpoints
        
        \pgfplotstablecreatecol[linear regression={y=y,x=x}]{regression}{\allpoints}
        \pgfmathsetmacro{\m}{\pgfplotstableregressiona}
        \pgfmathsetmacro{\b}{\pgfplotstableregressionb}
        
        \addplot[very thick, black, domain=0:1] {\m*x + \b};
        
        \end{axis}
        \end{tikzpicture}
        \caption{ScienceQA}
        \label{fig:improvement_corr:scienceqa}
    \end{subfigure}
    \begin{subfigure}[t]{.495\textwidth}
        \centering
        \scalefont{0.8}
        \begin{tikzpicture}
        \begin{axis}[
          width=\linewidth,
          width=7cm, height=4.5cm,
          xmin=0.4, xmax=1.0,
          ymin=0, ymax=0.1,
          xlabel={$\overline{\rho}$},
          ylabel={$\Delta A_{\mathrm{MV}}(16)$},
          xticklabel style={font=\scriptsize},
          legend style={at={(0.5,1.15)}, anchor=south, legend columns=2, nodes={scale=0.8, transform shape},legend cell align={left}},
          ymajorgrids=true,
          grid style=dashed,
        ]
        
        \addplot[
          only marks, mark=o, mark size=3.2pt,
          draw=black, fill=ETHGray, line width=0.6pt
        ] coordinates {(0.5395095389761869, 0.03720833333333334)};
        
        \addplot[
          only marks, mark=square*, mark size=3.2pt,
          draw=black, fill=ETHPurple, line width=0.6pt
        ] coordinates {(0.6055330133638783, 0.03066666666666662)};
        
        \addplot[
          only marks, mark=triangle*, mark size=3.6pt,
          draw=black, fill=ETHRed, line width=0.6pt
        ] coordinates {(0.7371498700649902, 0.019750000000000156)};
        
        \addplot[
          only marks, mark=diamond*, mark size=3.6pt,
          draw=black, fill=ETHBronze, line width=0.6pt
        ] coordinates {(0.4474045359139685, 0.05079166666666668)};
        
        \addplot[
          only marks, mark=otimes*, mark size=3.6pt,
          draw=black, fill=ETHGreen, line width=0.6pt
        ] coordinates {(0.6977003903938003, 0.02616666666666656)};
        
        \addplot[
          only marks, mark=star, mark size=2.6pt,
          draw=black, fill=ETHPetrol, line width=0.6pt
        ] coordinates {(0.7350099377528356, 0.0189583333333333)};
        
        \addplot[
          only marks, mark=pentagon*, mark size=3.6pt,
          draw=black, fill=ETHBlue, line width=0.6pt
        ] coordinates {(0.7875507779823582, 0.017458333333333242)};
        
        \pgfplotstableread{
        x   y
0.5395095389761869 0.03720833333333334
0.6055330133638783 0.03066666666666662
0.7371498700649902 0.019750000000000156
0.4474045359139685 0.05079166666666668
0.6977003903938003 0.02616666666666656
0.7350099377528356 0.0189583333333333
0.7875507779823582 0.017458333333333242
        }\allpoints
        
        \pgfplotstablecreatecol[linear regression={y=y,x=x}]{regression}{\allpoints}
        \pgfmathsetmacro{\m}{\pgfplotstableregressiona}
        \pgfmathsetmacro{\b}{\pgfplotstableregressionb}
        
        \addplot[very thick, black, domain=0:1] {\m*x + \b};
        
        \end{axis}
        \end{tikzpicture}
        \caption{MMStar}
        \label{fig:improvement_corr:mmstar}
    \end{subfigure}
    \begin{subfigure}[t]{.495\textwidth}
        \centering
        \scalefont{0.8}
        \begin{tikzpicture}
        \begin{axis}[
          width=\linewidth,
          width=7cm, height=4.5cm,
          xmin=0.4, xmax=1.0,
          ymin=0, ymax=0.1,
          xlabel={$\overline{\rho}$},
          xticklabel style={font=\scriptsize},
          legend style={at={(0.5,1.15)}, anchor=south, legend columns=2, nodes={scale=0.8, transform shape},legend cell align={left}},
          ymajorgrids=true,
          grid style=dashed,
        ]
        
        \addplot[
          only marks, mark=o, mark size=3.2pt,
          draw=black, fill=ETHGray, line width=0.6pt
        ] coordinates {(0.5277817141225475, 0.04673913043478267)};
        
        \addplot[
          only marks, mark=square*, mark size=3.2pt,
          draw=black, fill=ETHPurple, line width=0.6pt
        ] coordinates {(0.5907682313722583, 0.02934782608695652)};
        
        \addplot[
          only marks, mark=triangle*, mark size=3.6pt,
          draw=black, fill=ETHRed, line width=0.6pt
        ] coordinates {(0.7035297873481616, 0.02546583850931672)};
        
        \addplot[
          only marks, mark=diamond*, mark size=3.6pt,
          draw=black, fill=ETHBronze, line width=0.6pt
        ] coordinates {(0.45313883487187423, 0.04479813664596277)};
        
        \addplot[
          only marks, mark=otimes*, mark size=3.6pt,
          draw=black, fill=ETHGreen, line width=0.6pt
        ] coordinates {(0.6303151478420292, 0.02111801242236022)};
        
        \addplot[
          only marks, mark=star, mark size=2.6pt,
          draw=black, fill=ETHPetrol, line width=0.6pt
        ] coordinates {(0.7397389017794682, 0.01622670807453408)};
        
        \addplot[
          only marks, mark=pentagon*, mark size=3.6pt,
          draw=black, fill=ETHBlue, line width=0.6pt
        ] coordinates {(0.7755993933051787, 0.025931677018633637)};
        
        \pgfplotstableread{
        x   y
0.5277817141225475 0.04673913043478267
0.5907682313722583 0.02934782608695652
0.7035297873481616 0.02546583850931672
0.45313883487187423 0.04479813664596277
0.6303151478420292 0.02111801242236022
0.7397389017794682 0.01622670807453408
0.7755993933051787 0.025931677018633637
        }\allpoints
        
        \pgfplotstablecreatecol[linear regression={y=y,x=x}]{regression}{\allpoints}
        \pgfmathsetmacro{\m}{\pgfplotstableregressiona}
        \pgfmathsetmacro{\b}{\pgfplotstableregressionb}
        
        \addplot[very thick, black, domain=0:1] {\m*x + \b};
        
        \end{axis}
        \end{tikzpicture}
        \caption{MMMU}
        \label{fig:improvement_corr:mmmu}
    \end{subfigure}
    \caption{Majority voting improvement $\Delta A_{\mathrm{MV}}(16)$ versus average pairwise accuracy correlation ($\overline{\rho}$). Consistent with theory, stronger dependency (i.e., higher $\overline{\rho}$) corresponds to smaller gains from majority voting.}
    \label{fig:improvement_corr:all}
\end{figure}
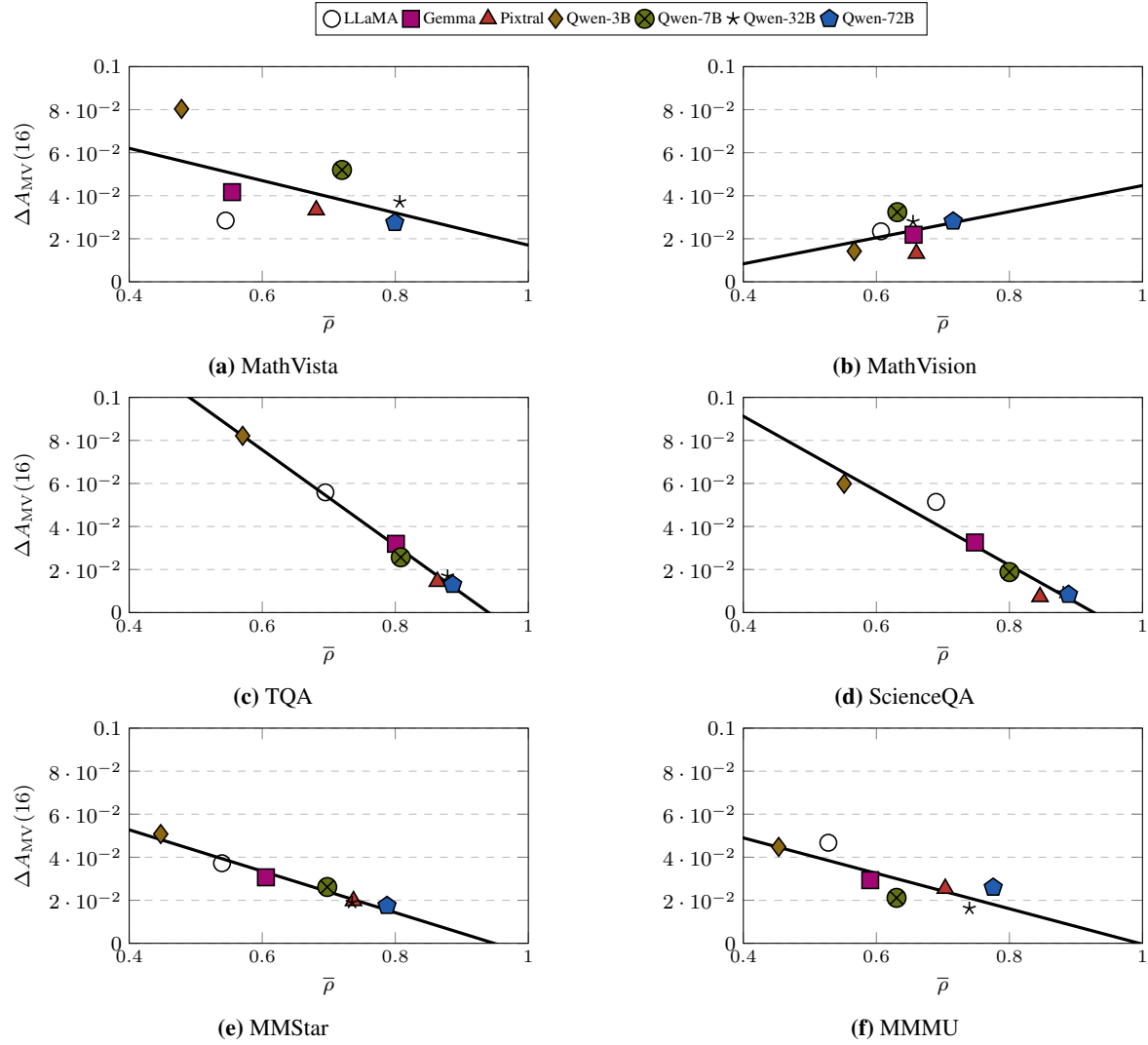

%% file: figures/assumption_all.tex
\begin{figure}[!ht]
    \centering
    \begin{subfigure}{.48\textwidth}
        \centering
        \scalefont{0.8}
        \begin{tikzpicture}
        \begin{axis}[
            width=7cm, height=4.5cm,
            ylabel={$\widetilde{H}_u$ / Accuracy},
            xlabel near ticks,
            ylabel near ticks,
            x tick label style={rotate=30, anchor=east},
            xtick=data,
            xticklabel style={font=\scriptsize},
            symbolic x coords={Qwen-3B, LLaMA, Pixtral, Gemma, Qwen-7B, Qwen-32B, Qwen-72B},
            legend style={at={(0.5,1.15)}, anchor=south, legend columns=7, nodes={scale=0.8, transform shape}},
            ymajorgrids=true,
            grid style=dashed,
        ]
        \addplot+[line width=0.3mm, mark size=1pt, color=ETHBlue] coordinates {
(Qwen-3B, 0.5194444444444444)
(LLaMA, 0.5203703703703704)
(Pixtral, 0.560300925925926)
(Gemma, 0.6503472222222222)
(Qwen-7B, 0.7208333333333333)
(Qwen-32B, 0.7857638888888888)
(Qwen-72B, 0.805787037037037)
        };
        \addlegendentry{Accuracy}
        \addplot+[line width=0.3mm, mark size=1pt, color=ETHPurple] coordinates {
(Qwen-3B, 0.5610214442606392)
(LLaMA, 0.477100606953641)
(Pixtral, 0.43812847503851154)
(Gemma, 0.31909536228134294)
(Qwen-7B, 0.28832551989156546)
(Qwen-32B, 0.18202658995058338)
(Qwen-72B, 0.1707680344509399)
        };
        \addlegendentry{$\widetilde{H}_u$}
        \end{axis}
        \end{tikzpicture}
        \caption{MathVista}
        \label{fig:hvsacc:mathvista}
    \end{subfigure}
    \begin{subfigure}{.48\textwidth}
        \centering
        \scalefont{0.8}
        \begin{tikzpicture}
        \begin{axis}[
            width=7cm, height=4.5cm,
            ylabel={$\widetilde{H}_u$ / Accuracy},
            xlabel near ticks,
            ylabel near ticks,
            xtick=data,
            x tick label style={rotate=30, anchor=east},
            xticklabel style={font=\scriptsize},
            symbolic x coords={Qwen-3B, LLaMA, Pixtral, Qwen-7B, Gemma, Qwen-32B, Qwen-72B},
            legend style={at={(0.5,1.15)}, anchor=south, legend columns=7, nodes={scale=0.8, transform shape}},
            ymajorgrids=true,
            grid style=dashed,
        ]
        \addplot+[line width=0.3mm, mark size=1pt, color=ETHBlue] coordinates {
(Qwen-3B, 0.2227072454308094)
(LLaMA, 0.23413022193211488)
(Pixtral, 0.25199902088772846)
(Qwen-7B, 0.30177056135770236)
(Gemma, 0.3183746736292428)
(Qwen-32B, 0.38797323759791125)
(Qwen-72B, 0.42889197127937334)
        };
        \addlegendentry{Accuracy}
        \addplot+[line width=0.3mm, mark size=1pt, color=ETHPurple] coordinates {
(Qwen-3B, 0.6510414917202695)
(LLaMA, 0.5986063905973528)
(Pixtral, 0.5624257395543766)
(Qwen-7B, 0.5188211468461807)
(Gemma, 0.44281018585924997)
(Qwen-32B, 0.4183311082343801)
(Qwen-72B, 0.3803737151357536)
        };
        \addlegendentry{$\widetilde{H}_u$}
        \end{axis}
        \end{tikzpicture}
        \caption{MathVision}
        \label{fig:hvsacc:mathvision}
    \end{subfigure}
    \begin{subfigure}{.48\textwidth}
        \centering
        \scalefont{0.8}
        \begin{tikzpicture}
        \begin{axis}[
            width=7cm, height=4.5cm,
            ylabel={$\widetilde{H}_u$ / Accuracy},
            xlabel near ticks,
            ylabel near ticks,
            xtick=data,
            x tick label style={rotate=30, anchor=east},
            xticklabel style={font=\scriptsize},
            symbolic x coords={Qwen-3B, LLaMA, Pixtral, Qwen-7B, Gemma, Qwen-32B, Qwen-72B},
            legend style={at={(0.5,1.15)}, anchor=south, legend columns=7, nodes={scale=0.8, transform shape}},
            ymajorgrids=true,
            grid style=dashed,
        ]
        \addplot+[line width=0.3mm, mark size=1pt, color=ETHBlue] coordinates {
(Qwen-3B, 0.6085235920852359)
(LLaMA, 0.7040715372907154)
(Pixtral, 0.7734208523592085)
(Qwen-7B, 0.785007610350076)
(Gemma, 0.7885844748858448)
(Qwen-32B, 0.8305745814307458)
(Qwen-72B, 0.8452435312024353)
        };
        \addplot+[line width=0.3mm, mark size=1pt, color=ETHPurple] coordinates {
(Qwen-3B, 0.41193610185881857)
(LLaMA, 0.2786298609137217)
(Pixtral, 0.17296221031584375)
(Qwen-7B, 0.1661429419207665)
(Gemma, 0.11334712800159344)
(Qwen-32B, 0.1021751962285444)
(Qwen-72B, 0.09323779493803595)
        };
        \end{axis}
        \end{tikzpicture}
        \caption{TQA}
        \label{fig:hvsacc:tqa}
    \end{subfigure}
    \begin{subfigure}{.48\textwidth}
        \centering
        \scalefont{0.8}
        \begin{tikzpicture}
        \begin{axis}[
            width=7cm, height=4.5cm,
            ylabel={$\widetilde{H}_u$ / Accuracy},
            xlabel near ticks,
            ylabel near ticks,
            xtick=data,
            x tick label style={rotate=30, anchor=east},
            xticklabel style={font=\scriptsize},
            symbolic x coords={Qwen-3B, LLaMA, Pixtral, Qwen-7B, Gemma, Qwen-32B, Qwen-72B},
            legend style={at={(0.5,1.15)}, anchor=south, legend columns=7, nodes={scale=0.8, transform shape}},
            ymajorgrids=true,
            grid style=dashed,
        ]
        \addplot+[line width=0.3mm, mark size=1pt, color=ETHBlue] coordinates {
(Qwen-3B, 0.6666769955379277)
(LLaMA, 0.7784147248388696)
(Pixtral, 0.7831866633614278)
(Qwen-7B, 0.7975954387704511)
(Gemma, 0.7783217649975211)
(Qwen-32B, 0.8420922161626178)
(Qwen-72B, 0.8464303420922161)
        };
        \addplot+[line width=0.3mm, mark size=1pt, color=ETHPurple] coordinates {
(Qwen-3B, 0.4333736318931796)
(LLaMA, 0.27218115603337417)
(Pixtral, 0.23063060535913832)
(Qwen-7B, 0.18840331958236753)
(Gemma, 0.1473228008112682)
(Qwen-32B, 0.11511642065829153)
(Qwen-72B, 0.10251710898100455)
        };
        \end{axis}
        \end{tikzpicture}
        \caption{ScienceQA}
        \label{fig:hvsacc:scienceqa}
    \end{subfigure}
    \begin{subfigure}{.48\textwidth}
        \centering
        \scalefont{0.8}
        \begin{tikzpicture}
        \begin{axis}[
            width=7cm, height=4.5cm,
            ylabel={$\widetilde{H}_u$ / Accuracy},
            xlabel near ticks,
            ylabel near ticks,
            xtick=data,
            x tick label style={rotate=30, anchor=east},
            xticklabel style={font=\scriptsize},
            symbolic x coords={Qwen-3B, LLaMA, Pixtral, Qwen-7B, Gemma, Qwen-32B, Qwen-72B},
            legend style={at={(0.5,1.15)}, anchor=south, legend columns=7, nodes={scale=0.8, transform shape}},
            ymajorgrids=true,
            grid style=dashed,
        ]
        \addplot+[line width=0.3mm, mark size=1pt, color=ETHBlue] coordinates {
(Qwen-3B, 0.41225)
(LLaMA, 0.460875)
(Pixtral, 0.5035)
(Qwen-7B, 0.5677083333333334)
(Gemma, 0.5339583333333333)
(Qwen-32B, 0.5633750000000001)
(Qwen-72B, 0.6255833333333334)
        };
        \addplot+[line width=0.3mm, mark size=1pt, color=ETHPurple] coordinates {
(Qwen-3B, 0.52351072403509)
(LLaMA, 0.4293925809755607)
(Pixtral, 0.36267603965471623)
(Qwen-7B, 0.2935391288906788)
(Gemma, 0.23967371219214964)
(Qwen-32B, 0.23833517106448646)
(Qwen-72B, 0.19023091761638874)
        };
        \end{axis}
        \end{tikzpicture}
        \caption{MMStar}
        \label{fig:hvsacc:mmstar}
    \end{subfigure}
    \begin{subfigure}{.48\textwidth}
        \centering
        \scalefont{0.8}
        \begin{tikzpicture}
        \begin{axis}[
            width=7cm, height=4.5cm,
            ylabel={$\widetilde{H}_u$ / Accuracy},
            xlabel near ticks,
            ylabel near ticks,
            xtick=data,
            x tick label style={rotate=30, anchor=east},
            xticklabel style={font=\scriptsize},
            symbolic x coords={Qwen-3B, LLaMA, Pixtral, Qwen-7B, Gemma, Qwen-32B, Qwen-72B},
            legend style={at={(0.5,1.15)}, anchor=south, legend columns=7, nodes={scale=0.8, transform shape}},
            ymajorgrids=true,
            grid style=dashed,
        ]
        \addplot+[line width=0.3mm, mark size=1pt, color=ETHBlue] coordinates {
(Qwen-3B, 0.37406832298136644)
(LLaMA, 0.42872670807453417)
(Pixtral, 0.4764751552795031)
(Qwen-7B, 0.5052795031055901)
(Gemma, 0.5249223602484472)
(Qwen-32B, 0.5903726708074535)
(Qwen-72B, 0.6418478260869565)
        };
        \addplot+[line width=0.3mm, mark size=1pt, color=ETHPurple] coordinates {
(Qwen-3B, 0.5870964701703758)
(LLaMA, 0.49615147426068584)
(Pixtral, 0.40994098682755964)
(Qwen-7B, 0.39733350724525984)
(Gemma, 0.3201634394253336)
(Qwen-32B, 0.270197942542754)
(Qwen-72B, 0.24305671238425539)
        };
        \end{axis}
        \end{tikzpicture}
        \caption{MMMU}
        \label{fig:hvsacc:mmmu}
    \end{subfigure}
    \caption{Correlation between normalized entropy $\widetilde{H}_u$ and accuracy across models on six benchmarks, supporting the Entropy–Accuracy Monotonicity assumption (\Cref{ass:entropy-accuracy}).}
    \label{fig:hvsacc:all}
\end{figure}
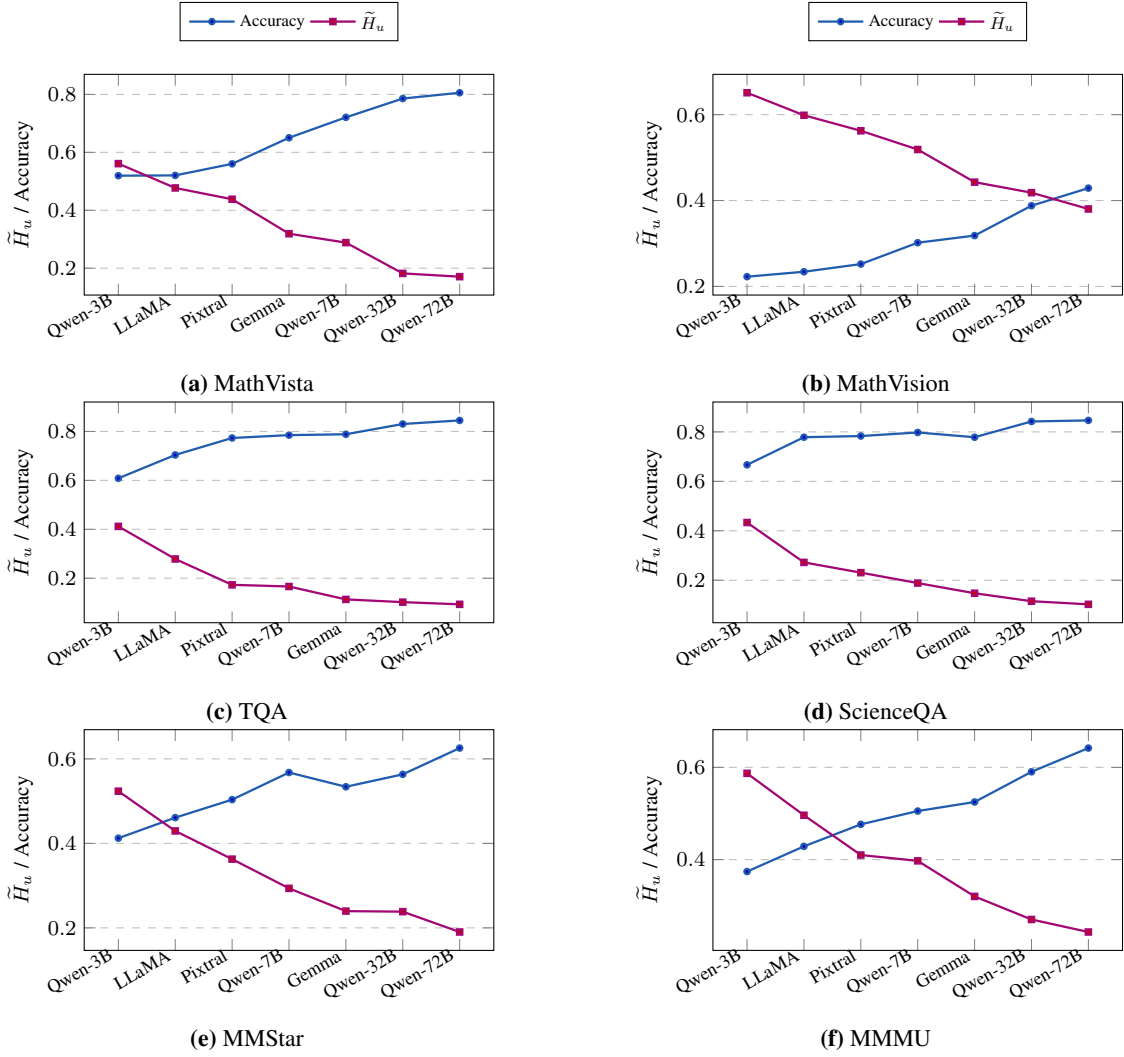

%% file: figures/appendix/ensemble_all.tex
\begin{table}[ht]
\small
\centering
\caption{\textbf{Robustness to Ensemble Composition.} Performance comparison (\%) of majority voting and ETTC across all pairwise and triplet combinations of the Qwen-2.5-VL family on MathVista. \textbf{Min/Max} denote the performance of the worst and best individual models in the ensemble. ETTC consistently outperforms the best individual model (Max) and voting, particularly in heterogeneous ensembles where weak models (e.g., 3B) degrade voting performance.}
\begin{tabular}{l|ccccc}
\toprule
\midrule
\textbf{Combination} & \textbf{Min.} & \textbf{Max.} & \textbf{Avg.} & \textbf{Voting} & \textbf{ETTC} \\
\midrule
3B, 7B & 51.94 & 72.08 & 62.01 & 69.81 & \textbf{79.26} \\
3B, 32B & 51.94 & 78.58 & 65.26 & 72.04 & \textbf{83.33} \\
3B, 72B & 51.94 & 80.58 & 66.26 & 73.15 & \textbf{84.81} \\
7B, 32B & 72.08 & 78.58 & 75.33 & 81.48 & \textbf{82.78} \\
7B, 72B & 72.08 & 80.58 & 76.33 & 82.22 & \textbf{84.63} \\
32B, 72B & 78.58 & 80.58 & 79.58 & \textbf{84.44} & 84.26 \\
3B, 7B, 32B & 51.94 & 78.58 & 67.53 & 81.30 & \textbf{82.41} \\
3B, 7B, 72B & 51.94 & 80.58 & 68.20 & 82.22 & \textbf{84.44} \\
3B, 32B, 72B & 51.94 & 80.58 & 70.37 & 83.70 & \textbf{84.63} \\
7B, 32B, 72B & 72.08 & 80.58 & 77.08 & 83.70 & \textbf{84.26} \\
\midrule
\bottomrule
\end{tabular}
\label{tab:ensemble_all}
\end{table}

%% file: figures/appendix/results_thinking.tex
\begin{table}[ht]
\small
\centering
\caption{\textbf{Generalization to Thinking LLMs.} Performance of ETTC versus majority voting on text-only reasoning benchmarks (ARC-Easy, MMLU-Pro) using Qwen-3-Thinking models. ETTC consistently outperforms voting and the best individual model (Max) across diverse ensemble sizes, confirming that entropy-based selection remains effective for pure language reasoning.}
\begin{tabular}{l|l|ccccc}
\toprule
\midrule
\textbf{Dataset} & \textbf{Models} & \textbf{Min.} & \textbf{Max.} & \textbf{Avg.} & \textbf{Voting} & \textbf{ETTC} \\
\midrule
\multirow{2}{*}{ARC-Easy} 
& 4B, 30B & 0.9599 & 0.9714 & 0.9656 & 0.9769 & \textbf{0.9878} \\
& 4B, 235B & 0.9599 & 0.9772 & 0.9686 & 0.9769 & \textbf{0.9878} \\
& 30B, 235B & 0.9714 & 0.9772 & 0.9743 & 0.9874 & \textbf{0.9895} \\
& 4B, 30B, 235B & 0.9599 & 0.9772 & 0.9695 & 0.9891 & \textbf{0.9899} \\
\midrule
\multirow{2}{*}{MMLU-Pro} 
& 4B, 30B & 0.8116 & 0.9412 & 0.8764 & 0.8934 & \textbf{0.9408} \\
& 4B, 235B & 0.8116 & 0.9431 & 0.8773 & 0.8979 & \textbf{0.9467} \\
& 30B, 235B & 0.9412 & 0.9431 & 0.9421 & 0.9504 & \textbf{0.9519} \\
& 4B, 30B, 235B & 0.8116 & 0.9431 & 0.8986 & 0.9482 & \textbf{0.9482} \\
\midrule
\bottomrule
\end{tabular}
\label{tab:thinking_llm}
\end{table}

%% file: figures/results_supervised.tex
\begin{table}[!ht]
\small
\centering
\caption{Evaluation results across datasets for \textbf{Similar Size Models} and \textbf{Same Family Models}. Columns show the average single-model accuracy (Average), Voting, (unsupervised) ETTC, and supervised variant of ETTC.}
\begin{tabular}{lcccccccc}
\toprule
\midrule
\multirow{2}{*}{\textbf{Accuracy \%}} 
& \multicolumn{4}{c}{\textbf{Similar Size Models}} 
& \multicolumn{4}{c}{\textbf{Same Family Models}} \\
\cmidrule(lr){2-5} \cmidrule(lr){6-9}
& Avg. & Voting & ETTC & Sup.~ETTC$_{\Delta}$ 
& Avg. & Voting & ETTC & Sup.~ETTC$_{\Delta}$ \\
\midrule
MathVista   & 61.30 & 68.33 & 75.93 & \increase{79.63}{3.70}
            & 70.80 & 83.15 & 84.44 & \increase{84.81}{0.37} \\
MathVision  & 27.66 & 32.05 & 35.57 & \increase{36.62}{1.05}
            & 33.53 & 41.32 & 44.84 & \increase{46.34}{1.50} \\
TQA         & 76.28 & 83.65 & 83.90 & \increase{84.14}{0.24}
            & 76.73 & 84.90 & 86.70 & \increase{86.70}{0.00} \\
ScienceQA   & 78.44 & 85.52 & 85.28 & \increase{85.97}{0.69}
            & 78.82 & 84.04 & 85.03 & \increase{86.07}{1.04} \\
MMStar      & 51.65 & 59.27 & 60.07 & \increase{60.67}{0.60}
            & 54.22 & 61.00 & 63.73 & \increase{65.07}{1.34} \\
MMMU        & 48.39 & 53.66 & 58.63 & \increase{59.01}{0.38}
            & 52.79 & 58.63 & 65.34 & \increase{66.46}{1.12} \\
\midrule
\textbf{Average}
            & 57.29 & 63.75 & 66.56 & \increase{67.67}{1.11}
            & 61.15 & 68.84 & 71.68 & \increase{72.58}{0.90} \\
\midrule
\bottomrule
\end{tabular}
\label{tab:entropyvsmv:merged}
\end{table}